\documentclass[twoside]{article}
\usepackage[accepted]{aistats2026}
\usepackage{times}  
\usepackage{helvet}  
\usepackage{courier}  
\usepackage[hyphens]{url}  
\usepackage{graphicx} 
\usepackage[round]{natbib}

\usepackage{caption} 
\usepackage{refcount}
\usepackage{algpseudocode}

\usepackage{dblfloatfix}
\usepackage{algorithm}
\usepackage[colorlinks=true]{hyperref}

\usepackage{tikz}

\usepackage{enumitem}
\setlistdepth{9}
\setlist[itemize,1]{label=$\bullet$}
\setlist[itemize,2]{label=$+$}
\setlist[itemize,3]{label=$--$}
\setlist[itemize,4]{label=$\star$}
\setlist[itemize,5]{label=$\circ$}
\setlist[itemize,6]{label=$\bullet$}
\setlist[itemize,7]{label=$\bullet$}
\setlist[itemize,8]{label=$\bullet$}
\setlist[itemize,9]{label=$\bullet$}
\renewlist{itemize}{itemize}{9}

\usepackage{times}
\usepackage{comment}

\usepackage{amsmath,amsfonts,bm}

\def\eqref#1{equation~\ref{#1}}

\def\1{\bm{1}}

\DeclareMathAlphabet{\mathsfit}{\encodingdefault}{\sfdefault}{m}{sl}
\SetMathAlphabet{\mathsfit}{bold}{\encodingdefault}{\sfdefault}{bx}{n}

\usepackage{url}
\usepackage{graphicx}
\usepackage{dcolumn}
\usepackage{bm}
\usepackage{amsmath}
\usepackage{amssymb} %
\usepackage{amsfonts}
\usepackage{nameref}
\usepackage{mathtools}
\newcommand{\mathdash}{\relbar\mkern-9mu\relbar}

\usepackage{bm}

\usepackage{algorithm}

\usepackage{color,xcolor}

\let\emptyset\varnothing

\newcommand{\cind}{\mathrel{\perp\mspace{-9mu}\perp}}
\newcommand{\notcind}{\mathrel{\,\not\!\cind}}
\usepackage{amsthm}
\newtheorem{theorem}{Theorem}
\newtheorem{corollary}{Corollary}[theorem]
\newtheorem{lemma}[theorem]{Lemma}
\newtheorem{remark}{Remark}
\newtheorem{assumption}{Assumption}

\newtheorem{proposition}{Proposition}
\newtheorem{definition}{Definition}[section]
\makeatletter
\def\munderbar#1{\underline{\sbox\tw@{$#1$}\dp\tw@\z@\box\tw@}}
\makeatother

\newcommand{\be}{\begin{equation}}
\newcommand{\ee}{\end{equation}}
\newcommand{\bea}{\begin{equation*}\begin{aligned}}
\newcommand{\eea}{\end{aligned}\end{equation*}}

\begin{document}
\twocolumn[
\aistatstitle{Causal Additive Models with Unobserved Causal Paths and Backdoor Paths}

\aistatsauthor{ Thong Pham \And Takashi Nicholas Maeda \And  Shohei Shimizu}

\aistatsaddress{Shiga University \\ RIKEN AIP \And Gakushuin University \\ RIKEN AIP \And The University of Osaka \\ Shiga University \\ RIKEN AIP} ]

\begin{abstract}%
Causal additive models provide a tractable yet expressive framework for causal discovery in the presence of hidden variables. When unobserved backdoor or causal paths exist between two variables, their causal relationship is often unidentifiable under existing theories. We establish sufficient conditions under which causal directions can be identified in many such cases. These conditions rely on new characterizations of regression sets to determine independence among regression residuals and conditional independencies among observed variables. Building on these results, we introduce a search algorithm that incorporates these innovations and prove its soundness and completeness. Empirical evaluations demonstrate its competitive performance against state-of-the-art methods.
\end{abstract}

\section{INTRODUCTION}

Discovering causal structure from observational data is a fundamental challenge across the sciences, with implications for reliable prediction~\citep{causality_for_ML}, decision-making~\citep{causal_decision_making}, and scientific understanding~\citep{glymour2019review}. A central approach in this area is to model functional relationships between variables together with stochastic noise, which allows for distinguishing causal directions under suitable assumptions~\citep{causal_functional}. Causal Additive Models (CAMs)~\citep{Buhlmann14CAM} form a particularly influential family of such models, as they capture nonlinear dependencies while remaining amenable to analysis. Due to their tractability, they have been studied considerably and have many practical applications in machine learning~\citep{pmlr-v162-budhathoki22a,yokoyama_2025}. 

When there are hidden variables, the problem of identifying the causal graph in CAMs remains relatively underexplored, limiting its practical adoption in settings where hidden variables are almost always present. Although causal discovery with hidden variables can be handled in full generality using the Fast Causal Inference (FCI) framework~\citep{Spirtes2001}, it is natural to expect that we can do better than FCI in certain aspects by exploiting specific properties of CAMs.

\cite{maeda21a} identify cases in which parent–child relationships in CAMs can be determined in the presence of unobserved variables by analyzing independencies and dependencies among certain regression residuals. This result is significant and relies on specific properties of CAMs, since parent–child relationships are unidentifiable in the FCI framework. Although their approach accommodates hidden variables, it requires the absence of unobserved backdoor or causal paths. When such paths exist, the parent–child relationship is deemed unidentifiable.

Given a target variable $x_j$, \cite{Schultheiss_2024} provide sufficient conditions for identifying a causally well-specified set of observed variables, i.e., a set for which the causal effect of intervening on $x_j$ is identifiable in CAMs with unobserved variables. Although they do not discuss causal search as an application, their framework can be used to identify causal directions in CAMs.

Focusing on CAMs with unobserved variables, we show that a) the parent–child relationship is identifiable in certain cases even in the presence of an unobserved backdoor or causal path, thereby extending~\cite{maeda21a}, and b) some causal directions beyond those identified by~\cite{Schultheiss_2024} are identifiable. 

Using our theory, we show that in some cases, the parent–child relationship between $x_i$ and $x_j$ is identifiable even when they share a common hidden parent $U$. This configuration, known as a \emph{bow}, is notoriously difficult for causal discovery because the hidden confounding introduced by $U$ obscures the parent–child relationship between $x_i$ and $x_j$~\citep{wang_2023,ashman2023causal}. To our knowledge, no prior work has established this kind of identifiability without additional assumptions on the graphical structure.

A high-level summary of our main contributions is as follows.
\begin{itemize}
\item By characterizing the regression sets used in determining independence, we show that the causal direction or parent–child relationship between a pair of variables can sometimes be identified using independence between the residuals, even when there are unobserved backdoor paths. We provide examples where the parent–child relationship in a bow is fully identified using the regression approach alone.
\item We introduce sufficient conditions, which combine conditional independence between the original variables with independence between regression residuals, to identify causal directions or parent–child relationships in pairs with unobserved backdoor or causal paths. We present examples showing that this hybrid approach can identify parent–child relationships in bow patterns and causal directions in pairs with unobserved causal paths, cases where neither regression alone nor conditional independence methods, such as the FCI framework, can.
\item We introduce the CAM-UV-X algorithm that incorporates the above innovations and also addresses a previously overlooked limitation of the CAM-UV algorithm in identifying causal relationships when all backdoor and causal paths are observable. We prove the soundness and completeness of CAM-UV-X in identifying certain causal patterns. We have released the code for the algorithm\footnote{\url{https://github.com/thongphamthe/CAM-UV-X}}.
\end{itemize}
The paper is organized as follows. We provide the background in Section~\ref{sec:prelim}. New identifiability results are presented in Section~\ref{sec:theory}. Our proposed search method CAM-UV-X is described in Section~\ref{sec:algorithms}. Numerical experiments are provided in Section~\ref{sec:experiment}. Conclusions are given in Section~\ref{sec:concluding}. 
\section{PRELIMINARIES}\label{sec:prelim}
\subsection{The causal model}
We assume the following causal additive model with unobserved variables as in~\cite{maeda21a}.
Let $X = \{x_i\}$ and $U = \{u_i\}$ be the sets of observable and unobservable variables, respectively. $G = (V,E)$ is the DAG with the vertex set $V = \{v_i\} = X \cup U$ and the edge set $E \subseteq V \times V$. The data generation model is
\begin{equation}
v_i = \sum_{j \in P_i}f_{j}^{(i)}(x_j) + \sum_{k\in R_i}f_{k}^{(i)}(u_k) + n_i,\label{eq:CAM_UV}
\end{equation}
where $P_i = \{j \mid (i,j) \in E \land x_j \in X\}$ is the set of observable direct causes of $v_i$, $R_i = \{k \mid (i,k) \in E \land u_k \in U\}$ is the set of unobservable direct causes of $v_i$, $f_{j}^{(i)}$ is a non-linear function, and $n_i$ is the external noise at $v_i$. 

We assume the following Causal Faithfulness Condition (CFC): 
\begin{assumption}\label{assumption:CFC}
Any conditional independence on $V$  that is not entailed by the d-separation criterion on $G$ 
 does not hold. 
\end{assumption}
CFC is standard in causal discovery methods~\citep{Spirtes2001}, and implies Assumption~1 of~\cite{maeda21a} (see Appendix~\ref{appendix:assumption}).

Let $\mathcal{G}$ be the class of GAM regression functions~\citep{hastie_GAM}. For a function $G_i\in \mathcal{G}$, $G_i(N) = \sum_{x_m\in N} g_{i,m}(x_m)$ where each $g_{i,m}(x_m)$ is a nonlinear function of $x_m$. As in~\cite{maeda21a}, we assume that the data-generating process (DGP) satisfies the following residual-faithfulness assumption with respect to $\mathcal{G}$:
\begin{assumption}\label{assumption:function_class}
For any $G_i,G_j\in \mathcal{G}$, any $x_i,x_j \in X$, any external noise $n_k$, and any $M,N \subseteq X$,  $(n_k \notcind x_i - G_i(M))\land (n_k \notcind x_j - G_j(N)) \Rightarrow x_i -  G_i(M)\notcind x_j - G_j(N)$.
\end{assumption}
The intuition is that, without such a restriction, a property identified using \(\mathcal{G}\) may reflect either the underlying causal structure of the DGP or merely an artifact of the model class. Assumption~\ref{assumption:function_class} rules out this ambiguity by requiring that whenever two residuals obtained using functions in \(\mathcal{G}\) both retain signal from the same exogenous noise variable, they must be statistically dependent. Analogous assumptions are common in functional causal discovery~\citep{ashman2023causal}.

\subsection{Unobserved backdoor paths and unobserved causal paths}
Unobserved backdoor paths (UBPs) and unobserved causal paths (UCPs) play central roles in the theory of causal additive models with unobserved variables. We introduce the concepts of UBPs/UCPs with respect to a set $X' \subseteq X$, extending the definitions of~\cite{maeda21a}, which were originally given for $X' = X$.

\begin{definition}[Unobserved Causal Path]
\label{def:UCP}
Consider a set of variables $X' \subseteq X$. A path in $G$ is called an unobserved causal path from $x_i$ to $x_j$ with respect to $X'$ if and only if it has the form $x_i \rightarrow \cdots \rightarrow v_k \rightarrow x_j$ with $x_i,x_j\in X'$ and $v_k \notin X'$. When omitted, $X'$ is taken to be $X$, the full set of observable variables. We use the term UCP between $x_i$ and $x_j$ to denote the existence of a UCP in either direction, from $x_i$ to $x_j$ or from $x_j$ to $x_i$.
\end{definition}
The path $x_4 \rightarrow U_3 \rightarrow x_5$ in Fig.~\ref{fig:illustrative_1}c is a UCP from $x_4$ to $x_5$. The path $x_1 \rightarrow U_1 \rightarrow x_3 \rightarrow x_2$ in Fig.~\ref{fig:illustrative_1}d is not a UCP between $x_1$ and $x_2$, but \emph{is} a UCP with respect to the set $X' =X \setminus \{x_3\}$. 

\begin{definition}[Unobserved Backdoor Path]
\label{def:UBP}
Consider a set of variables $X' \subseteq X$. For $x_i,x_j\in X'$, a path in $G$ is called an unobserved backdoor path between $x_i$ and $x_j$ with respect to $X'$ if and only if it is of the form $x_i \leftarrow v_{k}\leftarrow \cdots  \leftarrow v_c \rightarrow \cdots \rightarrow v_l \rightarrow x_j$ with $v_k,v_l \notin X'$. When omitted, $X'$ is taken to be $X$. \end{definition}

In other words, a UBP is a backdoor path between $x_i$ and $x_j$ where both parents of the endpoints on the path are not in $X'$.

If a path is a UBP/UCP between $x_i$ and $x_j$ with respect to some $S \subseteq X$, then it is also a UBP/UCP with respect to any $S'\subseteq S$ that contains $x_i$ and $x_j$.

Consider Fig.~\ref{fig:illustrative_1}a. The path $x_3 \leftarrow U_2 \rightarrow x_2$ is a UBP between $x_3$ and $x_2$ with respect to $X$, since both parents of $x_3$ and $x_2$ on the path, namely $U_2$, are not in $X$. The path $x_1 \leftarrow U_1 \rightarrow x_3 \rightarrow x_2$ is \emph{not} a UBP between $x_1$ and $x_2$ with respect to $X$, since the parent of $x_2$ on the path, namely $x_3$, is in $X$. However, it \emph{is} a UBP with respect to the set $X' = X \setminus \{x_3\}$. 
 

\begin{figure}[!ht]
\centering
\includegraphics[width = \columnwidth]{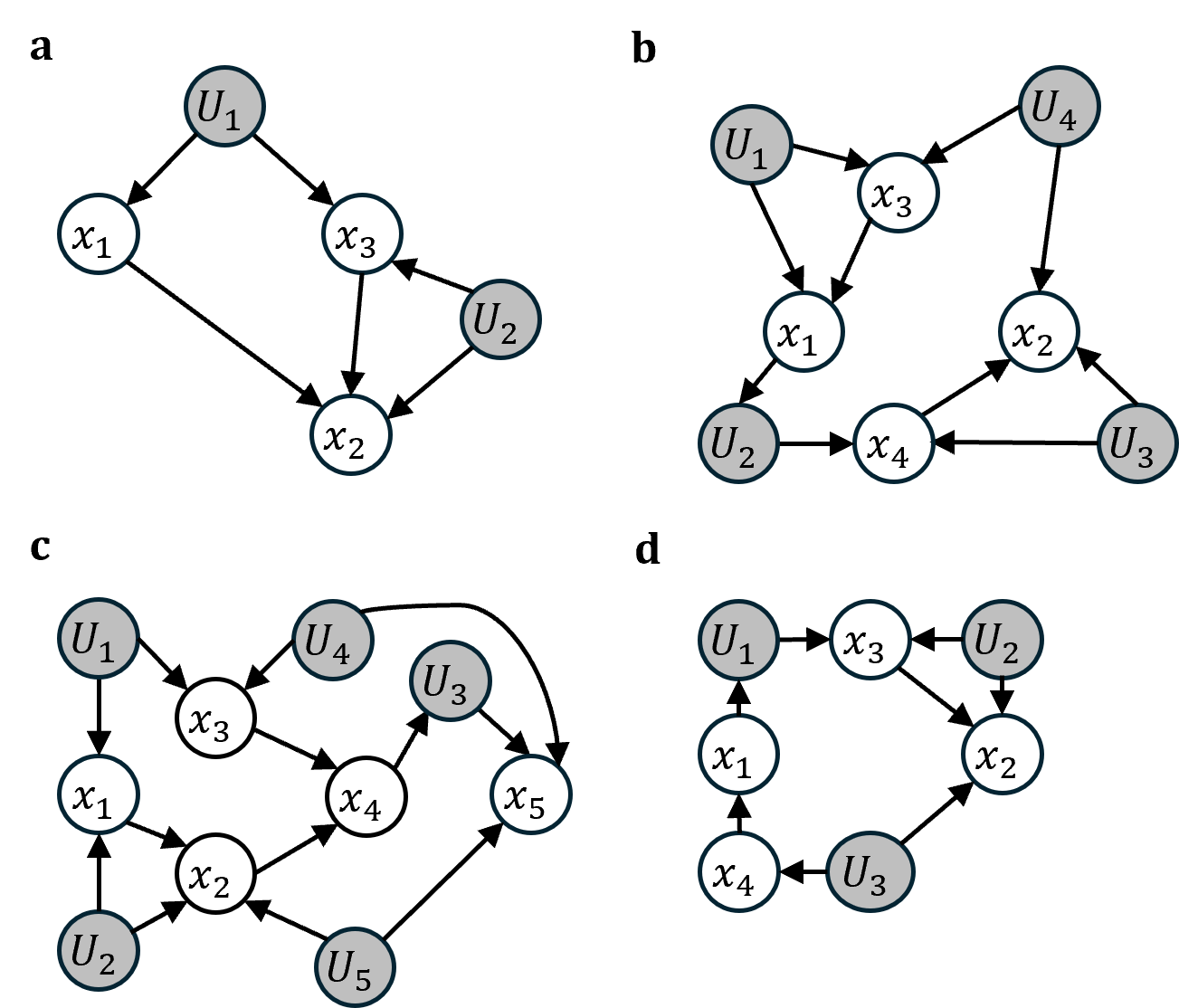}
\caption{Examples of identifying causal relationships in the presence of UBPs/UCPs.\label{fig:illustrative_1}}
\end{figure}
\subsection{Visible and invisible pairs}
Based on Lemmas~1-3 of~\cite{maeda21a}, we introduce the concepts of visibility/invisibility with respect to a set $X' \subseteq X$, which are fundamental to our theory and our proposed CAM-UV-X. The proofs of all lemmas in this section proceed by the same arguments as those of their corresponding lemmas, which were established for the case $X' = X$. We therefore omit the details.
\begin{lemma}[visible parent]\label{lemma:visible_parent}
Consider $X' \subseteq X$ and $x_i,x_j \in X'$. $x_j$ is a parent of $x_i$ and there is no UBP or UCP between $x_j$ and $x_i$ with respect to $X'$ if and only if
\begin{align}
\forall G_1,G_2\in \mathcal{G},M \subseteq X' \setminus \{x_i,x_j\}, N \subseteq X' \setminus \{x_j\}:\nonumber \\ 
x_i - G_1(M) \notcind x_j - G_2(N), \label{eq:visible_parent_1}\\
\exists G_1,G_2\in \mathcal{G},M \subseteq X' \setminus \{x_i\}, N \subseteq X' \setminus \{x_i,x_j\}:\nonumber\\
x_i - G_1(M) \cind x_j - G_2(N). \label{eq:visible_parent_2}
\end{align}
In this case, we refer to $x_j$ as a visible parent of $x_i$ with respect to $X'$. When omitted, $X'$ is taken to be $X$.
\end{lemma}
The regression functions $G_1$ and $G_2$ in $\mathcal{G}$ are used only as analytical tools; they need not correspond to the true causal function $f_j^{(i)}$ in Eq.~(\ref{eq:CAM_UV}). In practice, the function $G_i$ is fitted from observed data by GAM regression~\citep{hastie_GAM}.

\textbf{Example 1.} In Fig.~\ref{fig:illustrative_1}a, $x_j = x_1$ is a visible parent of $x_i = x_2$, since there are no UBPs/UCPs between $x_1$ and $x_2$. In particular, due to the direct effect of $x_1$ on $x_2$, there is no set $M\subseteq X\setminus \{x_1,x_2\} = \{x_3\}$ and no set $N \subseteq X\setminus \{x_1\} = \{x_2,x_3\}$ that can make the residuals $x_2 - G_1(M)$ and $x_1 - G_2(N)$ independent, i.e., Eq.~(\ref{eq:visible_parent_1}) is satisfied. Furthermore, there are some sets $M\subseteq X\setminus \{x_2\} = \{x_1,x_3\}$ and $N \subseteq X\setminus \{x_1,x_2\} = \{x_3\}$, in particular $M = \{x_1,x_3\}$ and $N = \emptyset$,  that can make the residuals $x_2 - G_1(M)$ and $x_1 - G_2(N)$ independent, i.e., Eq.~(\ref{eq:visible_parent_2}) is satisfied. 

Visible non-edges are defined as follows.
\begin{lemma}[visible non-edge]\label{lemma:visible_non_edge}
Consider $X'\subseteq X$ and $x_i,x_j \in X'$. There is no direct edge between $x_j$ and $x_i$ and there is no UBP or UCP between $x_j$ and $x_i$ with respect to $X'$ if and only if
\begin{align}
\exists G_1,G_2\in \mathcal{G},M \subseteq X' \setminus \{x_i, x_j\}, N \subseteq X' \setminus \{x_i,x_j\}:\nonumber\\
x_i - G_1(M) \cind x_j - G_2(N). \label{eq:non_edge}
\end{align}
In this case, we refer to $(x_i,x_j)$ as a visible non-edge with respect to $X'$. When omitted, $X'$ is taken to be $X$.
\end{lemma}
For a visible non-edge, independence can emerge even if $x_i$ and $x_j$ are excluded from the regression sets, e.g., the ranges of $M$ and $N$ in Eq.~(\ref{eq:non_edge}) are $X'\setminus \{x_i,x_j\}$. In contrast, for a visible edge $x_j \rightarrow x_i$, due to the direct effect of $x_j$ on $x_i$, $x_j$ must be included in the regression set of $x_i$ for independence to emerge (see Proposition~\ref{proposition:1} in Appendix~\ref{appendix:proof}).

\textbf{Example 2.} In Fig.~\ref{fig:illustrative_1}b, $(x_1,x_2)$ is a visible non-edge, since there are no UBPs/UCPs between $x_1$ and $x_2$. In this case, the non-edge $(x_1,x_2)$ is identifiable from data by checking Lemma~\ref{lemma:visible_non_edge}: there are some sets $M\subseteq X\setminus \{x_1,x_2\}$ and $N \subseteq X\setminus \{x_1,x_2\}$, in particular $M = \{x_4\}$ and $N = \{x_3\}$, that can make the residuals $x_2 - G_1(M)$ and $x_1 - G_2(N)$ independent, i.e., Eq.~(\ref{eq:non_edge}) is satisfied.

Finally, invisible pairs are defined as follows.
\begin{lemma}[invisible pairs]\label{lemma:invisible}
Consider $X' \subseteq X$ and $x_i,x_j \in X'$. There is a UBP/UCP between $x_j$ and $x_i$ with respect to $X'$ if and only if
\begin{align}
\forall M \subseteq X' \setminus \{x_i\}, N \subseteq X' \setminus \{x_j\},\forall G_{1},G_{2}\in \mathcal{G}:\nonumber \\
x_i - G_{1}(M) \notcind x_j - G_{2}(N). \label{eq:invisible}
\end{align}
In this case, we refer to $(x_i,x_j)$ as an invisible pair with respect to $X'$. When omitted, $X'$ is taken to be $X$.
\end{lemma}

\textbf{Example 3.} Consider  Fig.~\ref{fig:illustrative_1}a.  $(x_2,x_3)$ is invisible due to the presence of a UBP, namely $x_3 \leftarrow U_2 \rightarrow x_2$. In this case, the effects of $U_2$ on $x_2$ and $x_3$ cannot be removed through regressions, and thus there is no set $M \subseteq X\setminus \{x_2\}$ and no set $N \subseteq X \setminus \{x_3\}$ that can make $x_2 - G_1(M)$ and $x_3 - G_2(N)$ independent, i.e., Eq.~(\ref{eq:invisible}) is satisfied. While $(x_1,x_2)$ is visible with respect to $X$, it is invisible with respect to $X\setminus \{x_3\}$, because $x_1 \leftarrow U_1 \rightarrow x_3 \rightarrow x_2$ is a UBP with respect to $X \setminus \{x_3\}$.

The CAM-UV algorithm~\citep{maeda21a} is designed to detect visible edges and non-edges with respect to $X$ while marking invisible pairs as such. However, it cannot identify causal directions in invisible pairs.

\begin{remark}
The visibility and invisibility results with respect to \(X\) extend beyond CAM-UV to the separable observed/unobserved-parent model studied by~\cite{ashman2023causal}, but they do not extend to the general additive noise models (ANMs) of~\cite{Hoyer09NIPS}. The corresponding results with respect to a subset \(X' \subset X\) hold only in CAM-UV. See Appendix~\ref{appendix:visiblity_scope} for further discussion.
\end{remark}

\section{\MakeUppercase{Identifying causal relationships in invisible pairs}}\label{sec:theory}
We present conditions for identifying parent–child relationships or causal directions in invisible pairs. All omitted proofs can be found in Appendix~\ref{appendix:proof}.
\subsection{Utilizing independence between regression 
residuals}\label{sub_sec:independent}
In this section, new conditions are provided by characterizing the content of the regression sets $M$ and $N$ in Lemmas~\ref{lemma:visible_parent} and~\ref{lemma:visible_non_edge} with the following intuition.
\begin{proposition}\label{prop:invisibility_switch_meaning}
If $(x_i,x_j)$ is visible with respect to $X$ and invisible with respect to $X\setminus \{x_{k}\}$, $x_{k}$ must be a parent of $x_i$ or $x_j$.
\end{proposition}

Visibility with respect to $X$ and invisibility with respect to $X\setminus \{x_{k_q}\}$ can be confirmed using Lemmas~\ref{lemma:visible_parent},~\ref{lemma:visible_non_edge} and~\ref{lemma:invisible}. This allows us to conclude that $x_{k_q}$ must be a parent of $x_i$ or $x_j$, even when $(x_{k_q},x_i)$ and $(x_{k_q},x_j)$ are both invisible.  

The following lemma characterizes the regression sets in Eq.~(\ref{eq:visible_parent_2}). 
\begin{lemma}
\label{lemma:visible_edge_implication}
Consider distinct $x_i$, $x_j$, and $x_{k_1},\ldots,x_{k_m}$. Let $K = \{x_{k_1},\ldots,x_{k_m}\}$. $x_j$ is a visible parent of $x_i$, and for each $x_{k_q}$, $(x_i,x_j)$ is invisible with respect to $X \setminus \{x_{k_q}\}$ if and only if
\begin{align}
\text{For } q = 1,\ldots,m: \forall M \subseteq X \setminus \{x_i,x_{k_q}\}, \nonumber \\
N \subseteq X \setminus \{x_j,x_{k_q}\},
\forall G_{i}^1,G_{j}^1\in \mathcal{G}:\nonumber \\
x_i - G_{i}^{1}(M) \notcind x_j - G_{j}^1(N),\label{eq:visible_edge_implication_1}
\\
\forall M \subseteq X \setminus \{x_i,x_j\}, N \subseteq X \setminus \{x_j\},\forall G_{i}^1,G_{j}^1\in \mathcal{G}:\nonumber\\
x_i - G_{i}^{1}(M) \notcind x_j - G_{j}^1(N),\label{eq:visible_edge_implication_2}\\
\exists Q_{1},Q_{2} \subseteq K: Q_{1}\cup Q_{2} = K:\nonumber\\ 
\exists G_{i}^{2},G_{j}^{2}\in \mathcal{G},M, N \subseteq X \setminus \{x_i,x_j\}\setminus K:\nonumber\\ 
 x_i - G_{i}^{2}\left(M \cup \{x_j\} \cup Q_1\right) \cind x_j - G_{j}^{2}(N\cup Q_2 )\label{eq:visible_edge_implication_3}
\end{align}
are satisfied. Since $x_j$ is a parent of $x_i$, each $x_{k_q}$ is therefore an ancestor of $x_i$.
\end{lemma}
Eq.~(\ref{eq:visible_edge_implication_1}) means that for each node $x_{k_q}$ in $K$, the residuals cannot be independent if one excludes $x_{k_q}$ from $X$. Eq.~(\ref{eq:visible_edge_implication_3}) means that the residuals become independent when one includes each $x_{k_q}$ into the regression sets. Eqs.~(\ref{eq:visible_edge_implication_1}) and (\ref{eq:visible_edge_implication_3}) together imply that each $x_{k_q}$ must be on a UBP/UCP between $x_i$ and $x_j$ with respect to $X\setminus\{x_{k_q}\}$, and $x_{k_q}$ must be a parent of $x_i$ or a parent of $x_j$. Eq.~(\ref{eq:visible_edge_implication_2}) means that the residuals cannot be independent if one does not add $x_j$ to the regression set when regressing $x_i$. Eqs.~(\ref{eq:visible_edge_implication_2}) and~(\ref{eq:visible_edge_implication_3}) would imply that $x_j$ is a parent of $x_i$. Eqs.~(\ref{eq:visible_edge_implication_1}),~(\ref{eq:visible_edge_implication_2}), and~(\ref{eq:visible_edge_implication_3}) together give the lemma.
  
\textbf{Example~4.} The pair $(x_3,x_2)$ in Fig.~\ref{fig:illustrative_1}a is invisible, due to the presence of the UBP $x_3 \leftarrow U_2 \rightarrow x_2$. Thus, existing theories do not provide sufficient conditions for certifying causal relationships in $(x_3,x_2)$. If Eqs.~(\ref{eq:visible_edge_implication_1}),~(\ref{eq:visible_edge_implication_2}), and~(\ref{eq:visible_edge_implication_3}) hold with $x_i = x_2$, $x_j = x_1$, and $K =\{x_3\}$, Lemma~\ref{lemma:visible_edge_implication} can certify that $x_3$ is an ancestor of $x_2$. 


The full separability of causal effects in CAM-UV is crucial for our theory and may be the additional structural constraint that enables identification in cases not covered by~\cite{Schultheiss_2024}. Additionally, they aim to identify variables whose causal effect on a target variable is identifiable, which is generally a harder problem than just identifying the ancestors of the target. For example, \cite{Hoyer08IJAR} showed examples where the causal direction is identifiable, even when the causal effect is not.


For visible non-edges, Lemma~\ref{lemma:visible_non_edge_implication} characterizes the regression sets in Eq.~(\ref{eq:non_edge}).
\begin{lemma}
\label{lemma:visible_non_edge_implication}
Consider distinct $x_i$, $x_j$, and $x_{k_1},\ldots,x_{k_m}$. Let $K = \{x_{k_1},\ldots,x_{k_m}\}$. $(x_j,x_i)$ is a visible non-edge, and for each $x_{k_q}$, $(x_i,x_j)$ is invisible with respect to $X \setminus \{x_{k_q}\}$ if and only if Eq.~(\ref{eq:visible_edge_implication_1}) and 
\begin{align}
\exists Q_{1},Q_{2} \subseteq K: Q_{1}\cup Q_{2} = K:\nonumber\\ 
\exists G_{i}^{2},G_{j}^{2}\in \mathcal{G},M, N \subseteq X \setminus \{x_i,x_j\}\setminus K:\nonumber \\ 
 x_i - G_{i}^{2}\left(M \cup Q_1 \right) \cind x_j - G_{j}^{2}(N\cup Q_2), \label{eq:visible_non_edge_implication_1} 
\end{align}
\end{lemma}
are satisfied.

\textbf{Example 5.} Consider Fig.~\ref{fig:illustrative_1}b. Since $(x_3,x_1)$, $(x_3,x_2)$, $(x_4,x_1)$, and $(x_4,x_2)$ are invisible, existing theories cannot identify causal relationships in these pairs. If Eqs.~(\ref{eq:visible_edge_implication_1}) and~(\ref{eq:visible_non_edge_implication_1}) hold in the observed data with $x_i = x_1$, $x_j = x_2$, and $K = \{x_3,x_4\}$, Lemma~\ref{lemma:visible_non_edge_implication} can certify that $x_3$ is a parent of $x_1$ or $x_2$, and $x_4$ is a parent of $x_1$ or $x_2$. 

See Appendix~\ref{appendix:additional_results} for Lemma~\ref{lemma:visible_non_edge_implication_2}, which provides an alternative condition to Lemma~\ref{lemma:visible_non_edge_implication}.

The presence of a bow between two variables poses a major challenge for causal discovery. For example, in non-Gaussian linear models, the parent–child relationship within a bow cannot be determined without additional assumptions, such as adopting a canonical model where every hidden variable is parentless and acts as a confounder of at least two observed variables~\citep{Hoyer08IJAR,pmlr-v235-tramontano24a,10.1609/aaai.v38i18.30017}, or assuming that the number of hidden variables is known~\citep{NEURIPS2021_c0f6fb5d}. Without such assumptions, only the causal direction, i.e., the ancestor relationship, is identifiable~\citep{Salehkaleybar2020learning}. To our knowledge, no existing causal model can identify the parent–child relationship within a bow without imposing additional assumptions on the graph structure. 


Lemmas~\ref{lemma:visible_edge_implication},~\ref{lemma:visible_non_edge_implication}, or~\ref{lemma:visible_non_edge_implication_2} imply that $x_{k_q}$ must be a parent of $x_i$ or $x_j$, without specifying which. If we can rule out $x_{k_q}$ as a parent of $x_i$, for example, when $(x_{k_q},x_i)$ is a visible non-edge or $x_i \rightarrow x_{k_q}$ is a visible edge, then $x_{k_q}$ must be a parent of $x_j$. Even when both $(x_{k_q},x_i)$ and $(x_{k_q},x_j)$ are invisible and form bows, the regression approach can sometimes still resolve the relationship (see Appendix~\ref{appendix:example_parent_child_regression}). When it cannot, conditional independence among the original variables may provide the needed information.

\subsection{Utilizing conditional independence}\label{sub_sec:conditional_independent}
We provide sufficient conditions utilizing independence between regression residuals and conditional independence between the original variables to identify causal relationships in invisible pairs. We show some examples where the causal direction is identifiable by this hybrid approach, but not by the regression approach alone or conditional independence alone. 

We introduce the following lemma for identifying that some $x_j$ is not an ancestor of $x_i$.
\begin{lemma}
\label{lemma:y_structure}
If $x_k$ is an ancestor of $x_i$ and $x_k \cind x_j \mid x_i$, then 1) there is no backdoor path between $x_i$ and $x_j$, and 2) $x_j$ is not an ancestor of $x_i$.
\end{lemma}

Lemma~\ref{lemma:y_structure} is a direct consequence of standard d-separation reasoning~\citep{Spirtes2001} and not new. It is closely related to local causal-discovery arguments based on Y-structures~\citep{Mani2006UAI}. We therefore state it here only in the particular form needed for our setting.

Combining Lemma~\ref{lemma:y_structure} with independence conditions between regression residuals can give sufficient conditions to identify the causal direction in certain invisible pairs in the following corollary.
 
\begin{corollary}\label{coro:ancestorship}
Consider an invisible pair $(x_i,x_j)$, i.e., Eq.~(\ref{eq:invisible}) is satisfied. If  $x_k$ is an ancestor of $x_i$ and $x_k \cind x_j \mid x_i$, then $x_i$ is an ancestor of $x_j$.
\end{corollary}

\textbf{Example 6.} In Fig.~\ref{fig:illustrative_1}c, consider the invisible pair $(x_4,x_5)$. Since this pair is not on any backdoor path or causal path of any visible pair, Lemmas~\ref{lemma:visible_edge_implication}, ~\ref{lemma:visible_non_edge_implication}, or~\ref{lemma:visible_non_edge_implication_2} cannot identify the causal direction in this pair. Since $(x_4,x_5)$ is invisible, Eq.~(\ref{eq:invisible}) is satisfied for $x_i = x_4$ and $x_j = x_5$. From the visible pairs $(x_2,x_4)$ and $(x_3,x_4)$, $x_2$ and $x_3$ are identified as parents of $x_4$ by Lemma~\ref{lemma:visible_parent}. If Lemma~\ref{lemma:visible_non_edge_implication} is satisfied with $x_i = x_3$, $x_j = x_2$, and $K=\{x_1\}$, $x_1$ can be identified as a parent of $x_3$ or a parent of $x_2$, and thus is an ancestor of $x_4$. If $x_1 \cind x_5 \mid x_4$, $x_4$ can be identified as an ancestor of $x_5$ by applying Corollary~\ref{coro:ancestorship} with $x_k = x_1$.

\begin{remark} See Appendix~\ref{appendix:example_ancestor} for an example in which Corollary~\ref{coro:ancestorship} can identify the causal direction in an invisible pair, whereas neither the regression approach alone nor conditional independence alone, e.g., FCI, can. 
\end{remark}

Lemma~\ref{lemma:y_structure} can also identify parent–child relationships in invisible pairs, as in the following corollary. 
\begin{corollary}\label{coro:parentship}
Suppose that $x_k$ is a parent of $x_i$ or a parent of $x_j$ (e.g., by Lemmas~\ref{lemma:visible_edge_implication},~\ref{lemma:visible_non_edge_implication}, or~\ref{lemma:visible_non_edge_implication_2}). Furthermore, for a fourth variable $x_u$, if $x_u$ is an ancestor of $x_i$ and $x_u \cind x_k \mid x_i$, then $x_k$ is a parent of $x_j$.   
\end{corollary}

\textbf{Example 7.} Consider the bow $(x_2,x_3)$ in Fig.~\ref{fig:illustrative_1}d. Corollary~\ref{coro:ancestorship} cannot identify the causal direction in this pair. If Lemma~\ref{lemma:visible_non_edge_implication} is satisfied with $x_i = x_1$, $x_j = x_2$, and $K = \{x_3\}$, one can identify that $x_3$ is a parent of $x_1$ or a parent of $x_2$. Note that this is a hard case where the regression approach alone cannot resolve. The edge $x_4\rightarrow x_1$ is visible and thus can be identified. If $x_4 \cind x_3 \mid  x_1$ holds, applying Corollary~\ref{coro:parentship} with $x_k = x_3$, $x_i = x_1$, $x_j = x_2$, and $x_u = x_4$ identifies $x_3$ as a parent of $x_2$.
\begin{remark}
This is an example in which regression and conditional independence together can identify the parent–child relationship in a bow that remains unidentifiable by regression alone.
\end{remark}

\begin{remark}
Corollary~\ref{coro:parentship} relies on the pair $(x_i,x_j)$ being visible with respect to $X$ but invisible with respect to $X \setminus \{x_k\}$. Consequently, it does not extend beyond CAM-UV. By contrast, Corollary~\ref{coro:ancestorship} extends beyond CAM-UV to a slightly larger class of models considered in~\cite{ashman2023causal}; see Appendix~\ref{appendix:visiblity_scope} for further discussion.
\end{remark}

\section{\MakeUppercase{Search methods}}\label{sec:algorithms}
In Section~\ref{subsection:cam_uv_limitation}, we discuss previously overlooked limitations of the CAM-UV algorithm, which can lead to incorrect identification of some visible pairs as invisible. Our proposed method is presented in Section~\ref{subsection:our_method}.

\subsection{Limitations of the CAM-UV algorithm}\label{subsection:cam_uv_limitation}

The CAM-UV algorithm is incomplete in identifying visible pairs and unsound in identifying invisible pairs. Soundness and completeness~\citep{Spirtes2001}  can be defined in our setting as follows. An algorithm is complete for visible edges if and only if every visible edge in the ground truth also appears as a visible edge in the algorithm’s output. Conversely, an algorithm is sound for visible edges if and only if every visible edge in the output also exists in the ground truth. Analogous definitions apply to visible non-edges and invisible pairs.

\textbf{Fig.~\ref{fig:illustrative_1}a} illustrates that CAM-UV might misclassify a visible edge as invisible, making it incomplete for visible edges and unsound for invisible pairs. Although the visible edge $x_1\rightarrow x_2$ is identifiable by Lemma~\ref{lemma:visible_parent}, the presence of the invisible parent $x_3$ of $x_2$ leads CAM-UV to wrongly label it as invisible. 

\textbf{Fig.~\ref{fig:illustrative_1}b} illustrates that CAM-UV might misclassify some visible non-edge as invisible, making it also incomplete for visible non-edges. The visible non-edge $(x_1,x_2)$ is identifiable by Lemma~\ref{lemma:visible_non_edge}. However, to block all backdoor and causal paths between $x_1$ and $x_2$, one must add $x_3$ and $x_4$ to the regression sets in Eq.~(\ref{eq:non_edge}). Furthermore, the pairs $(x_1,x_3)$, $(x_1,x_4)$, $(x_2,x_3)$, and $(x_2,x_4)$ are all invisible. In this case, CAM-UV will wrongly label the visible non-edge as invisible. 

See Appendix~\ref{appendix:cam_uv_step_by_step} for step-by-step executions of CAM-UV in these examples. 

\subsection{Proposed search method}\label{subsection:our_method}
Our proposed method CAM-UV-X, described in Algorithm~\ref{alg:CAM_UV_extended}, addresses the limitations of CAM-UV in identifying visible pairs, and leverages the new theory in this paper to infer parent–child relationships and causal directions in invisible pairs. 

CAM-UV-X additionally uses Lemmas~\ref{lemma:non_parent_condition} and~\ref{lemma:visible_non_edge_implication_2} in Appendix~\ref{appendix:additional_results}. Lemma~\ref{lemma:non_parent_condition} provides sufficient conditions to certify that $(x_i,x_j)$ is visible and $x_i$ is not a parent of $x_j$. If one can certify non-parent relations in both directions, then one can conclude that $(x_i,x_j)$ is a visible non-edge, i.e.,  a different route than using Lemma~\ref{lemma:visible_non_edge}. Lemma~\ref{lemma:visible_non_edge_implication_2} utilizes this new visible non-edge condition in the same way as Lemma~\ref{lemma:visible_non_edge_implication}.

The output of CAM-UV-X is $A$, $M_1,\ldots,M_p$, $H_{1},\ldots, H_{p}$, and $C_1,\ldots,C_p$. $A$ is the adjacency matrix over the observed variables. $A(i,j) = 1$ if $x_j$ is inferred to be a parent of $x_i$, $0$ if there is no directed edge from $x_j$ to $x_i$, and NaN (Not a Number) if $(x_i,x_j)$ is inferred to be invisible. $M_i$ is the set of ancestors of $x_i$ identified by Lemma~\ref{lemma:visible_edge_implication} and Corollary~\ref{coro:ancestorship}. $H_{i}$ is the set of nodes guaranteed not to be an ancestor of $x_i$, identified by, for example, Lemma~\ref{lemma:y_structure}. $C_k$ contains ordered pairs $[i,j]$ such that $x_k$ is a parent of $x_i$ or $x_k$ is a parent of $x_j$.  See Appendix~\ref{appendix:running_time} for the running time of CAM-UV-X.
\algrenewcommand\algorithmicrequire{\textbf{Input:}}
\algrenewcommand\algorithmicensure{\textbf{Output:}}
\begin{algorithm}[!ht]
\caption{CAM-UV-X}\label{alg:CAM_UV_extended}
\begin{algorithmic}[1]
\Require{$n\times p$ data matrix $X$ for $p$ observed variables, maximum number of parents $d$, significance level $\alpha$}
\Ensure{$A$, $\{M_1,\ldots,M_p\}$, $\{H_1,\ldots,H_p\}$, $\{C_1,\ldots,C_p\}$}

\State $A \gets$ CAM-UV($X,d,\alpha$)\;

\State Initialize $M_i \gets \emptyset$, $H_i \gets \emptyset$, $C_i \gets \emptyset$ for $i = 1,\ldots,p$

\State Find the set $S = \{ (i,j) \mid A(i,j) = A(j,i) = \text{NaN}\}$

 \For{each $(i,j) \in S, i < j$}
\State    \texttt{checkVisible}(i,j)
\EndFor
\State Find the set $I = \{ (i,j) \mid A(i,j) \ne \text{NaN}; A(j,i) \ne \text{NaN}\}$

\For{each $(i,j) \in I, i < j$}
  \State  \texttt{checkOnPath}(i,j)
\EndFor
\State Find the set $S = \{ (i,j) \mid A(i,j) = A(j,i) = \text{NaN}\}$

\For{each $(i,j) \in S$}
\State \texttt{checkCI}(i,j) 
\EndFor
\State \texttt{checkParentInvi}()
\end{algorithmic}
 \end{algorithm}

In line 1, CAM-UV is executed to obtain an initial estimate of $A$. The procedure~\texttt{checkVisible}, described in Algorithm~\ref{alg:checkVisible}, tests whether each NaN element in the current matrix $A$ can be converted to $1$, that is, a visible edge, or to $0$, i.e., a visible non-edge. 

\begin{algorithm}[!ht]
\caption{\texttt{checkVisible}}\label{alg:checkVisible}
\begin{algorithmic}[1]

\Require{indices $i$ and $j$}
\State $P_i \gets \{v\mid A(i,v) = 1\};\  P_j\gets \{v \mid A(j,v) = 1\}$

\State $Q \gets \{k \mid k\notin \{i,j\},\ A(j,k) = \text{NaN} \text{ or } A(i,k) = \text{NaN} \} \cup P_i \cup P_j$ 
    
\State $iNotParent \gets False$; $jNotParent \gets False$
    
\For{each $M\subseteq Q$ and $N \subseteq Q$}
    
\State $e\gets\widehat{\text{p-HSIC}}\big(x_i - G_1(M),$
    $x_j - G_2(N)\big)$ 
    
\If{$e > \alpha$}
 \State   $A(i,j) \gets 0$; $A(j,i) \gets 0$; \Return
\Else
\State $a_1 \gets \widehat{\text{p-HSIC}}\big(x_i - G_1(M\cup \{x_j\}),$
    $x_j - G_2(N)\big)$
    
\State $a_2 \gets \widehat{\text{p-HSIC}}\big(x_i - G_1(M),$
    $x_j - G_2(N\cup \{x_i\})\big)$
    
\If{$a_1 > \alpha$}
    \State $iNotParent \gets True$
\EndIf
\If{$a_2 > \alpha$}
 \State $jNotParent \gets True$
\EndIf
\If{$iNotParent \land jNotParent$}
\State   $A(i,j) \gets 0$; $A(j,i) \gets 0$; \Return
\EndIf
\EndIf
\EndFor
\If{$iNotParent$}
\State  $A(i,j) \gets 1$; $A(j,i) \gets 0$
\EndIf
\If{$jNotParent$}
\State $A(j,i) \gets 1$; $A(i,j) \gets 0$
\EndIf     
\end{algorithmic}
\end{algorithm}

On line~6 of \texttt{checkVisible}, Lemma~\ref{lemma:visible_non_edge} is checked. To find the functions $G_i$, we use GAM regression implemented in pyGAM~\citep{Servan18pyGAM}. Eq.~(\ref{eq:non_edge}) is satisfied if the value $e$, calculated on line 5, is greater than $\alpha$. Here, $\widehat{\text{p-HSIC}}$ is the p-value of the gamma independence test based on Hilbert–Schmidt Independence Criterion~\citep{gretton_hsic}. If so, $(x_i,x_j)$ is concluded to be a visible non-edge. 

If the condition on line 6 is not satisfied, the p-HSIC values $a_1$ and $a_2$ are calculated. If $a_1 > \alpha$ in line 11, Eq.~(\ref{eq:not_parent_1}) is satisfied, and we conclude that $x_i$ is not a parent of $x_j$ in line 12, due to Lemma~\ref{lemma:non_parent_condition} (see Appendix~\ref{appendix:additional_results}). We proceed similarly to $a_2$. If $x_i$ is not a parent of $x_j$ and $x_j$ is not a parent of $x_i$, i.e., line 17 is executed, we conclude that $(x_i,x_j)$ is a visible non-edge, due to Lemma~\ref{lemma:non_parent_condition}. 

If line 22 is reached, Eq.~(\ref{eq:visible_parent_1}) is satisfied. Strictly speaking, we only check subsets in $Q$, not all subsets of $X$ as in Eq.~(\ref{eq:visible_parent_1}). Reassuringly, this is enough; see Appendix~\ref{appendix:proof_theorem}. Then, if $x_i$ is not a parent of $x_j$, i.e., the condition on line 22 is satisfied, or if $x_j$ is not a parent of $x_i$, i.e., the condition on line 25 is satisfied, we conclude that $(x_i,x_j)$ is a visible edge due to Lemma~\ref{lemma:visible_parent}. If the conditions in both line 22 and line 25 are not satisfied, we conclude that $(x_i,x_j)$ is invisible.  

The algorithm \texttt{CAM-UV-X} then invokes the procedure \texttt{checkOnPath}, described in Algorithm~\ref{alg:check_OnPath}, to check Eq.~(\ref{eq:visible_edge_implication_1}) for every visible pair $(x_i,x_j)$ and every $x_{k_q}$ whose parent–child relationship to $x_i$ or $x_j$ is invisible. If the equation is satisfied, i.e., $isOnPath$ is $True$ at line 10, $x_{k_q}$ is concluded to be a parent of $x_i$ or a parent of $x_j$, due to Lemmas~\ref{lemma:visible_edge_implication},~\ref{lemma:visible_non_edge_implication}, and~\ref{lemma:visible_non_edge_implication_2}.

\algnewcommand\algorithmicbreak{\textbf{break}}
\algnewcommand\Break{\State \algorithmicbreak}

\begin{algorithm}[!ht]
\caption{\texttt{checkOnPath}}\label{alg:check_OnPath}
\begin{algorithmic}[1]    
\Require{indices $i$, $j$}
\Ensure{Boolean value $isOnPath$}

\For{each $x_k \in X\setminus \{x_i,x_j\} \text{ that satisfies } A(i,k) = \text{NaN} \text{ or } A(j,k) = \text{NaN}$}

\State $isOnPath \gets True$

\For{each $M \subseteq X \setminus \{x_i,x_k\}$ and $N \subseteq X \setminus \{x_j,x_k\}$} 
\State $a\gets\widehat{\text{p-HSIC}}\big(x_i - G_1(M),x_j - G_2(N)\big)$

\If{$a > \alpha$} 
  \State $isOnPath \gets False$; \Break
\EndIf  
\EndFor
\If{$isOnPath$}
\State $C_k \gets C_k \cup \{[i,j]\} \cup \{[j,i]\}$

\If{A(i,j) = 1} 
\State $M_i \gets M_i \cup \{x_k\}$; $H_k \gets H_k \cup \{x_i\}$; 
$A(k,i) \gets 0$
\ElsIf{A(j,i) = 1}

\State $M_j \gets M_j \cup \{x_k\}$; $H_k \gets H_k \cup \{x_j\}$; $A(k,j) \gets 0$
\EndIf
\EndIf
\EndFor
\end{algorithmic}
\end{algorithm}

The procedure \texttt{checkCI}, described in Algorithm~\ref{alg:check_CI}, checks conditional independence of the form $x_k \cind x_j \mid x_i$ with $(x_i,x_j)$ being invisible, and $x_k$ being an ancestor of $x_i$. $\widehat{\text{p-CI}}$ is the p-value of some conditional independence test. Some examples are the conditional mutual information test based on nearest-neighbor estimator (CMIknn from~\cite{Runge_CMIknn}) and the conditional independence test based on Gaussian process regression and distance correlations (GPDC from~\cite{GPDC}). If conditional independence is satisfied (line 3), $x_i$ is an ancestor of $x_j$ due to Lemma~\ref{lemma:y_structure} and Corollary~\ref{coro:ancestorship}.
\begin{algorithm}[!ht]
\caption{\texttt{checkCI}}\label{alg:check_CI}
\begin{algorithmic}[1]
\Require{indices $i$ and $j$}
\For{each currently identified ancestor $x_k$ of $x_i$} 
\State $e \gets \widehat{\text{p-CI}}(x = x_k, y = x_j, z = x_i)$

\If{$e > \alpha$} 
  
  \State $H_i \gets H_i\cup \{x_j\}$;  $A(i,j) \gets 0$;
  $M_j \gets M_j \cup \{x_i\}$
\EndIf
\EndFor
\end{algorithmic}
\end{algorithm}

In \texttt{checkParentInvi}, described in Algorithm~\ref{alg:enforceConsistency}, for each ordered pair $[i,j] \in C_k$ such that $A(j,k) = \text{NaN}$, we check a) whether $x_k$ is not an ancestor of $x_i$ by checking whether $x_k$ is in $H_i$ and b) whether $x_k$ is not a parent of $x_i$ by checking whether $A(i,k) = 0$. If either case is true, we conclude that $x_k$ is a parent of $x_j$. The loop repeats until there is no new change.

\begin{algorithm}[!ht]
\caption{\texttt{checkParentInvi}}\label{alg:enforceConsistency}
\begin{algorithmic}[1]
    
\Repeat
\State $isChange \gets False$

\For{$k = 1,\ldots, p$}
    \For{each ordered pair $[i,j] \in C_k$}
        \If{$(A(j,k) = \text{NaN}) \land (x_k \in H_i \lor A(i,k) = 0$)}
           \State $A(j,k) \gets 1; A(k,j) \gets 0$ 
            
            \State $isChange \gets True$
            \EndIf
        \EndFor
\EndFor
\Until{$isChange = False$}
\end{algorithmic}
\end{algorithm}

To prove the correctness of CAM-UV-X in identifying visibilities, we need the following assumption.
\setcounter{assumption}{2}
\begin{assumption}\label{assumption:correct_regression}
For any $x_i,x_j\in X$, $M \subseteq X \setminus \{x_i\}$, $N \subseteq X\setminus \{x_j\}$, let $e = \widehat{\text{p-HSIC}}\big(x_i - G_i(M),x_j - G_j(N)\big)$, where $G_i,G_j$ are GAM regression functions fitted by pyGAM. For a given $\alpha$, the following equation holds: 
\begin{equation}
\exists G_1,G_2\in \mathcal{G}:
x_i - G_1(M) \cind x_j - G_2(N) \iff e > \alpha.\nonumber
\end{equation}
\end{assumption}
This assumption means that the pyGAM regression and the HSIC independence test can identify the regression functions in $\mathcal{G}$. Note that CAM-UV is not complete in identifying visible pairs, and is not sound in identifying invisible pairs, even with this assumption. The following theorem, whose proof can be found in Appendix~\ref{appendix:proof}, holds.
\begin{theorem}\label{theorem:CAM_UV_X_sound_complete}
With Assumptions~\ref{assumption:CFC},~\ref{assumption:function_class}, and~\ref{assumption:correct_regression}, CAM-UV-X is sound and complete in identifying visible edges, visible non-edges, and invisible pairs.
\end{theorem}

Note that Theorem~\ref{theorem:CAM_UV_X_sound_complete} does not assert soundness or completeness for identifying the causal direction within an invisible pair. Section~\ref{sec:theory} provides sufficient conditions for such identification, and therefore CAM-UV-X is sound whenever it orients an invisible pair. Completeness, however, remains open, even when restricting the class of graphs $G$ to exclude obvious non-identifiable cases such as a standalone bow. Indeed, there may exist bows that are resolvable in principle but not by the current theory. We are unaware of any such example, but we also cannot rule out this possibility, since that would require a characterization of the necessary conditions under which such graphical structures arise, which is currently unknown.

See Appendix~\ref{sec:related_works} for a discussion of related works.

\section{\MakeUppercase{Experiments}}\label{sec:experiment}
\subsection{Illustrative examples}\label{sec:experiment_illustrative}
We demonstrate that CAM-UV-X can address the limitations of CAM-UV discussed in the previous section by applying it to Figs.~\ref{fig:illustrative_1}a and~\ref{fig:illustrative_1}b. The DGP in Eq.~(\ref{eq:CAM_UV}) is defined as follows. The function $f_{j}^{(i)}(x_j)$ is set to $(x_j + a)^c + b$ with random coefficients $a$, $b$, and $c$. Note that this function belongs to GAMs. We use the same setting for $f_{k}^{(i)}$. The noise terms $n_i$ are Gaussian. For each graph, $100$ datasets, each of $500$ samples, were generated. We ran CAM-UV and CAM-UV-X with the significance level $\alpha = 0.1$. In CAM-UV-X, we used CMIknn from~\cite{Runge_CMIknn} as the conditional independence estimator. 

We measured the success rate of identifying the visible edge $x_1\rightarrow x_2$ and $x_3$ as an ancestor of $x_2$ in Fig.~\ref{fig:illustrative_1}a, and identifying the visible non-edge $(x_1,x_2)$ in Fig.~\ref{fig:illustrative_1}b. Figure~\ref{fig:illustrative_2} shows the results. CAM-UV-X successfully addresses the limitations of CAM-UV discussed in Section~\ref{subsection:cam_uv_limitation}.

\begin{figure}[!ht]
\centering
\includegraphics[width = \columnwidth]{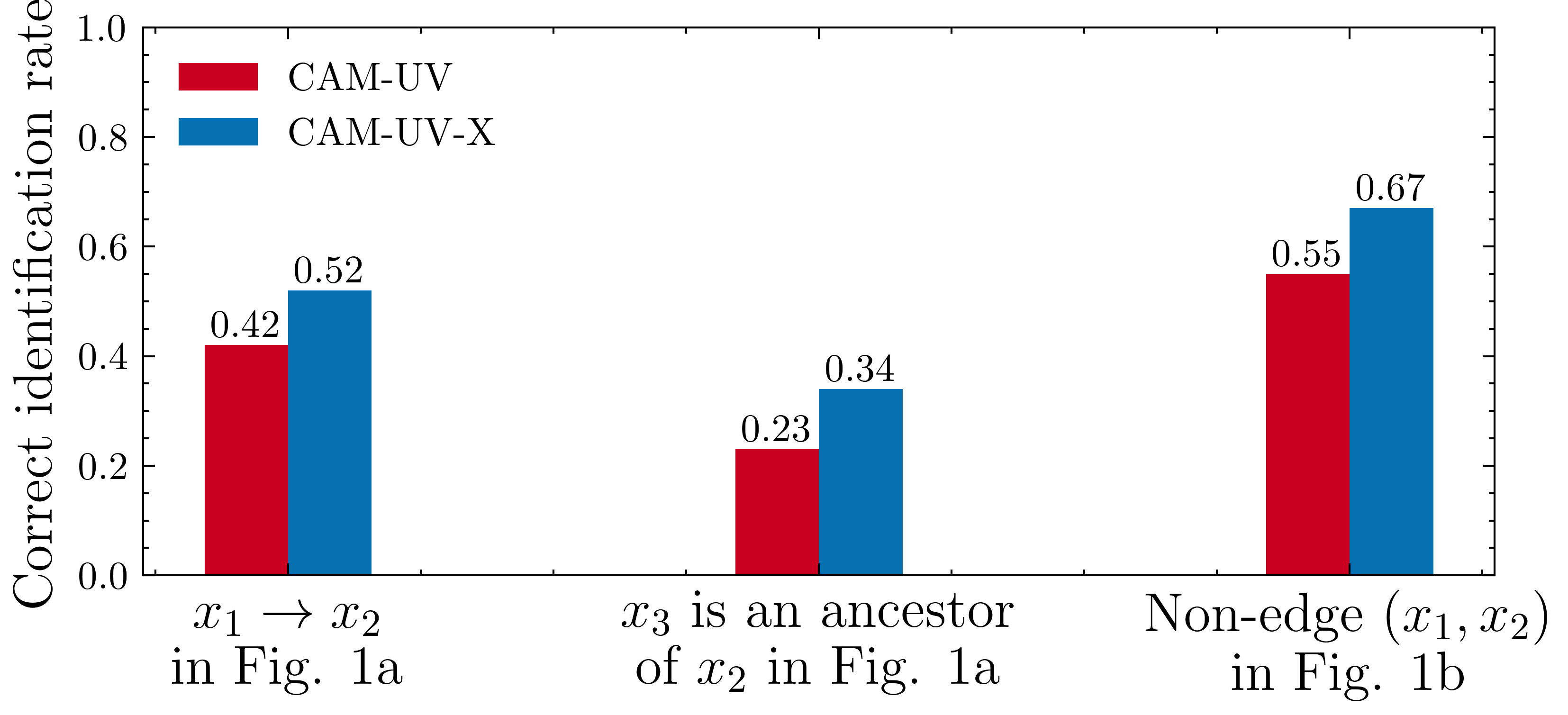}
\caption{Performance on the graphs in Figs.~\ref{fig:illustrative_1}a and~\ref{fig:illustrative_1}b.}\label{fig:illustrative_2}
\end{figure}
\subsection{Random graph experiment}
We investigate CAM-UV-X using simulated data generated from Eq.~(\ref{eq:CAM_UV}). We report the results for the Barab{\'{a}}si–Albert (BA) model~\citep{ba_model} here; results for the Erd\"{o}s-R\'{e}nyi (ER) model~\citep{erdos59a} are deferred to Appendix~\ref{appendix:experiment_details}.

We generate $50$ BA graphs with $40$ nodes, where each node has five children. For each graph, we create data using the same process as in Section~\ref{sec:experiment_illustrative}. We then randomly select $10$ variables and form the final dataset using only these $10$ variables. Each dataset contains $500$ samples. In addition to CAM-UV, we also include an adaptation of~\cite{Schultheiss_2024}, called S-B, as another baseline. See Appendix~\ref{appendix:experiment_details} for details.

We ran the algorithms with significance levels $\alpha = 0.05$, $0.1$, and $0.2$. We show results of two tasks: identifying the adjacency matrix and identifying ancestor relations between the observed variables, in Figs.~\ref{fig:BA_gaussian_combine}a and~\ref{fig:BA_gaussian_combine}b, respectively. See Appendix~\ref{appendix:experiment_details} for definitions of evaluation metrics used in each task. 

\begin{figure}[!h]
\centering
\includegraphics[width = \columnwidth]{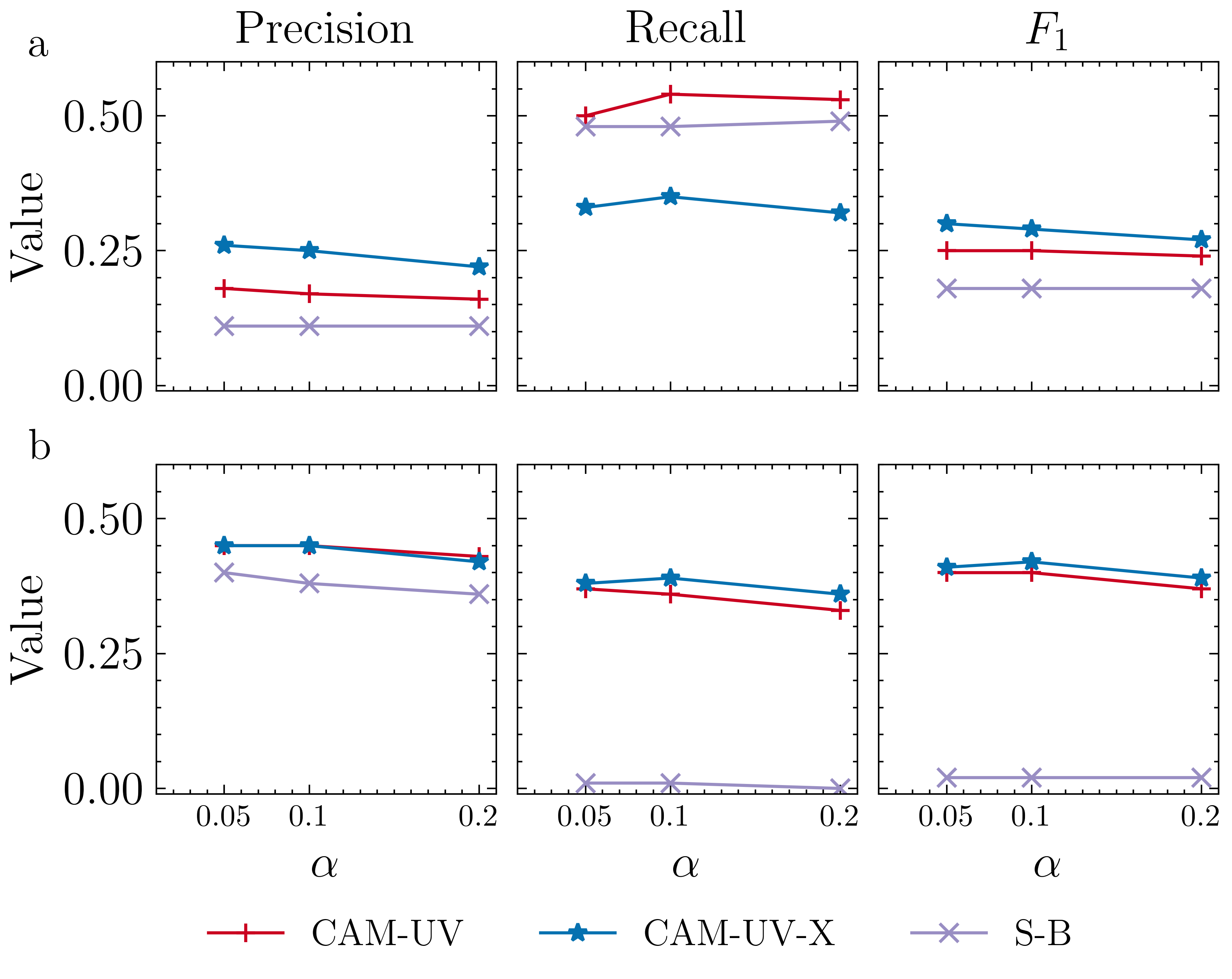}
\caption{Performance in BA random graphs with Gaussian noise. a: identifying the adjacency matrix, b: identifying ancestral relationships.}\label{fig:BA_gaussian_combine}
\end{figure}

For the identification of the adjacency matrix, CAM-UV-X achieved higher precision and F1 scores than CAM-UV. For the identification of ancestor relationships, CAM-UV-X achieved higher recall and F1 scores than CAM-UV. The remaining results in Appendix~\ref{appendix:ER_results} also show that CAM-UV-X compares favorably with the baselines.

\subsection{Real-world sociological data}
We present in Fig.~\ref{fig:status_attainment_model_result} the results of applying CAM-UV and CAM-UV-X to a sociological dataset~\citep{Shimizu11JMLR}. The data consisted of six observed variables: father's occupation ($x_1$), son's income ($x_2$), father's education ($x_3$), son's occupation ($x_4$), son's education ($x_5$), and number of siblings ($x_6$). 

 Compared to CAM-UV, CAM-UV-X pruned more spurious bidirected edges, namely edges that CAM-UV should have pruned if it were sound in identifying invisible pairs. For example, CAM-UV-X replaced the bidirected edge between $x_1$ and $x_4$ in the CAM-UV result by a directed edge from $x_1$ to $x_4$, which is consistent with the ground truth. 
 
 Furthermore, CAM-UV-X can provide more information for invisible pairs than CAM-UV. For example, for the pairs $(x_1,x_3)$ and $(x_2,x_4)$, while there is a UBP/UCP between these pairs, CAM-UV-X asserts that each pair is not adjacent. See Appendix~\ref{appendix:sociology_data} for details of the experiment settings. 

\begin{figure}[t]
\centering
\includegraphics[width = \columnwidth]{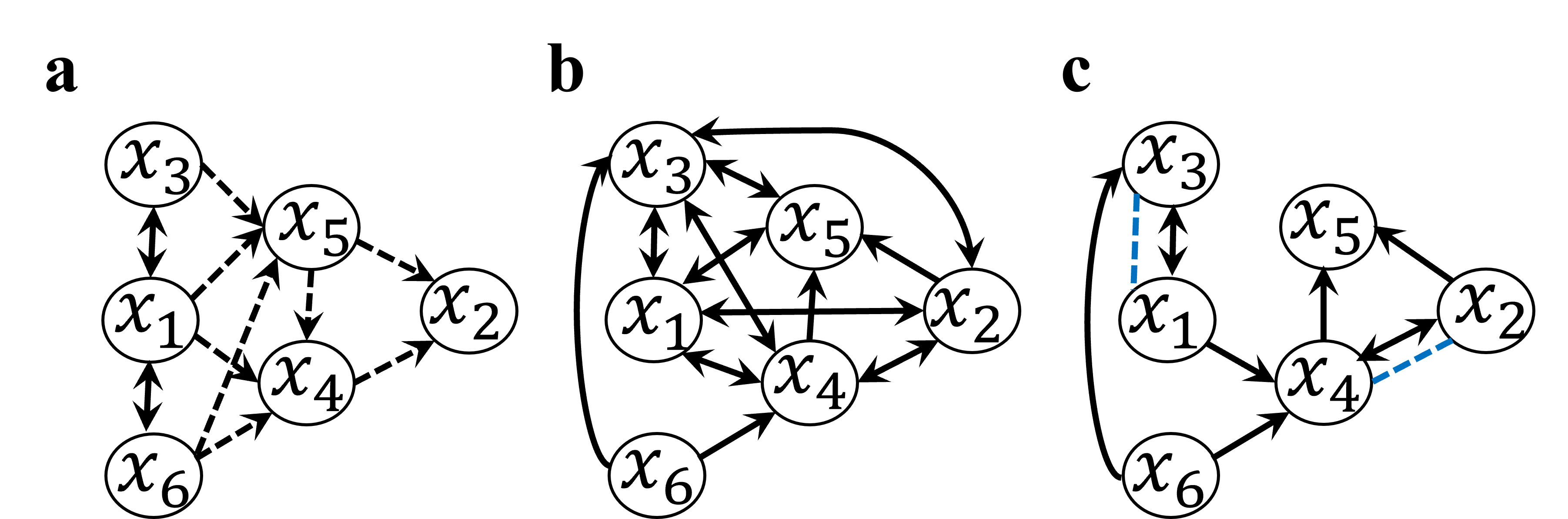}
\caption{Results on sociological data. \textbf{a}: ground truth based on domain knowledge~\citep{Duncan72book}. A bidirected edge indicates that the relationship is not modeled. A dashed directed edge represents an ancestor relationship. \textbf{b}: result by CAM-UV. A solid directed edge denotes a visible parent–child relationship (adjacency). An empty edge denotes a visible non-edge (non-adjacency). A bidirected edge denotes an invisible pair. \textbf{c}: result of CAM-UV-X. Solid directed/bidirected edges and empty edges have the same meaning as those of CAM-UV. A dashed undirected edge in blue, unique to CAM-UV-X and appearing only in an invisible pair connected by a bidirected edge, indicates non-adjacency.}\label{fig:status_attainment_model_result}
\end{figure}

\section{\MakeUppercase{Conclusions}}\label{sec:concluding}
We provided new identifiability results for causal additive models with hidden variables, including cases where bows are fully identified. Furthermore, we introduced the CAM-UV-X algorithm, proved its soundness and completeness, and demonstrated that it performs comparably to state-of-the-art methods.

Our current theory does not fully utilize all constraints implied by d-separation among the observed variables. A more refined characterization of the regression sets that are sufficient or necessary to certify visibility or invisibility could lead to more efficient algorithms. Furthermore, systematically investigating the gap between sufficient and necessary conditions in CAMs with hidden variables remains an important open problem. Finally, extending the use of conditional independence between observed variables and regression residuals to other classes of causal models offers a promising direction for future research.

\subsubsection*{Acknowledgements}
This work was partially supported by the Japan Science and Technology Agency (JST) under CREST Grant Number JPMJCR22D2 and by the Japan Society for the Promotion of Science (JSPS) under KAKENHI Grant Numbers  JP24K20741 and JP25K03084.


\section*{Checklist}

\begin{enumerate}

  \item For all models and algorithms presented, check if you include:
  \begin{enumerate}
    \item A clear description of the mathematical setting, assumptions, algorithm, and/or model. [Yes]
    \item An analysis of the properties and complexity (time, space, sample size) of any algorithm. [Yes]
    \item (Optional) Anonymized source code, with specification of all dependencies, including external libraries. [Yes]
  \end{enumerate}

  \item For any theoretical claim, check if you include:
  \begin{enumerate}
    \item Statements of the full set of assumptions of all theoretical results. [Yes]
    \item Complete proofs of all theoretical results. [Yes]
    \item Clear explanations of any assumptions. [Yes]     
  \end{enumerate}

  \item For all figures and tables that present empirical results, check if you include:
  \begin{enumerate}
    \item The code, data, and instructions needed to reproduce the main experimental results (either in the supplemental material or as a URL). [Yes]
    \item All the training details (e.g., data splits, hyperparameters, how they were chosen). [Yes]
    \item A clear definition of the specific measure or statistics and error bars (e.g., with respect to the random seed after running experiments multiple times). [Yes]
    \item A description of the computing infrastructure used (e.g., type of GPUs, internal cluster, or cloud provider). [Yes]
  \end{enumerate}

  \item If you are using existing assets (e.g., code, data, models) or curating/releasing new assets, check if you include:
  \begin{enumerate}
    \item Citations of the creator if your work uses existing assets. [Not Applicable]
    \item The license information of the assets, if applicable. [Yes]
    \item New assets either in the supplemental material or as a URL, if applicable. [Yes]
    \item Information about consent from data providers/curators. [Not Applicable]
    \item Discussion of sensible content if applicable, e.g., personally identifiable information or offensive content. [Not Applicable]
  \end{enumerate}

  \item If you used crowdsourcing or conducted research with human subjects, check if you include:
  \begin{enumerate}
    \item The full text of instructions given to participants and screenshots. [Not Applicable]
    \item Descriptions of potential participant risks, with links to Institutional Review Board (IRB) approvals if applicable. [Not Applicable]
    \item The estimated hourly wage paid to participants and the total amount spent on participant compensation. [Not Applicable]
  \end{enumerate}

\end{enumerate}

\clearpage

\appendix
\onecolumn
\thispagestyle{empty}
\aistatstitle{Causal Additive Models with Unobserved Causal Paths and Backdoor Paths: \\
Supplementary Materials}

\renewcommand\thefigure{\thesection.\arabic{figure}}    
\setcounter{figure}{0}  
\renewcommand\thetable{\thesection.\arabic{table}}    
\setcounter{table}{0} 
\section{\MakeUppercase{Assumptions of the CAM-UV model}}\label{appendix:assumption}

\subsection{CFC implies Assumption~1 of~\cite{maeda21a}} 
Assumption~1 of~\cite{maeda21a} is stated as the following Assumption~1b.

\textbf{Assumption 1b.}
If variables $v_i$ and $v_j$ have terms involving functions of the same external effect $n_k$, then $v_i$ and $v_j$ are mutually dependent, i.e., $(n_k \notcind v_i) \land (n_k \notcind v_j) \Rightarrow v_i \notcind v_j$.

In our paper, we assume the standard CFC and Assumption~\ref{assumption:function_class}. Since the proofs of Lemmas~\ref{lemma:visible_parent},~\ref{lemma:visible_non_edge}, and~\ref{lemma:invisible} in~\cite{maeda21a} rely crucially on Assumption~1b, we prove that CFC implies Assumption~1b as follows. 
\begin{proof}
Assume CFC. If variables $v_i$ and $v_j$ have terms involving functions of the same external effect $n_k$, there is a confounder between $v_i$ and $v_j$, or a directed path from $v_i$ to $v_j$, or a directed path from $v_j$ to $v_i$. This implies that $v_i$ and $v_j$ are not d-separated by the empty set. By CFC, this implies that $v_i$ is not independent of $v_j$. Thus, Assumption 1b is true under CFC. This means CFC implies Assumption 1b.
\end{proof}

\section{\MakeUppercase{Scope of the visibility and invisibility results}}\label{appendix:visiblity_scope}
\cite{ashman2023causal} extended Lemmas 1-3 of~\cite{maeda21a}, i.e., visibility/invisibility with respect to $X$, to a broader class of models where the causal effects between the observed and unobserved parents are separable, while the effects among the observed parents might not be separable:
\begin{equation}
v_i = f_i(P_i) + g_i(R_i) + n_i, \label{eq:ashman_model}
\end{equation}
where $P_i$ and $R_i$ are the sets of observed and unobserved parents of $v_i$, respectively. The functions $f_i$ and $g_i$ are non-linear functions. Equation~(\ref{eq:ashman_model}) is equivalent to Eq.~(5) of~\cite{ashman2023causal}.

Beyond the model in Eq.~(\ref{eq:ashman_model}), visibility/invisibility with respect to $X$ cannot be extended further to more general ANMs, as shown in the following example. Consider the model: $x_2 \rightarrow x_3$, $x_1 \rightarrow x_3$, $U \rightarrow x_3$, $U \rightarrow x_1$. $x_2$ and $U$ are external. The equation for $x_3$ is $x_3 = \sin(x_1 + x_2 + U) + n_3$. This is a more general ANM than the model in Eq.~(\ref{eq:ashman_model}), since the unobserved parent $U$ is not separable from the observed parents $x_1$ and $x_2$ in the equation of $x_3$.

In this example, there is no UBP/UCP between $x_2$ and $x_3$, and $x_2$ is a parent of $x_3$. However, there is no regression function $G$ of two variables $x_1$ and $x_2$ such that $x_3 - G(x_1,x_2)$ can be independent of $x_2$, due to the unobserved factor $U$ inside the non-separable function $sin(x_1 + x_2 + U)$. This means the lemma about visible parents does not hold.

Lemmas~\ref{lemma:visible_parent},~\ref{lemma:visible_non_edge}, and ~\ref{lemma:invisible} in this paper, however, cannot be extended to the model in Eq.~(\ref{eq:ashman_model}). These lemmas introduce the new concept of visibility/invisibility with respect to a subset $X' \subset X$, and thus require that the model class must be closed under reclassifying an observed variable as unobserved. CAM-UV satisfies this closure property, while the model in Eq.~(\ref{eq:ashman_model}) of~\cite{ashman2023causal} does not.

Therefore, all results in this paper that rely on invisibility with respect to a proper subset $X' \subset X$ can only hold in CAM-UV, not in the model of Eq.~(\ref{eq:ashman_model}).

On the other hand, since Corollary~\ref{coro:ancestorship} only requires invisibility with respect to $X$, it holds in the larger model of Eq.~(\ref{eq:ashman_model}).

\section{\MakeUppercase{Examples where the parent–child relationship in an invisible pair is identified by the regression approach}}\label{appendix:example_parent_child_regression}
\textbf{Fig.~\ref{fig:example_appendix}a}: There are two visible pairs in this example: $(x_2,x_3)$ is a visible non-edge and $x_1\rightarrow x_2$ is a visible edge. Applying Lemma~\ref{lemma:visible_edge_implication} to the visible edge $x_1 \rightarrow x_2$ with $K = \{x_3\}$, we get that $x_3$ is a parent of either $x_1$ or $x_2$. However, since $(x_2,x_3)$ is a visible non-edge, $x_3$ cannot be a parent of $x_2$. Therefore, $x_3$ can be identified as a parent of $x_1$ by the regression approach. 

\textbf{Fig.~\ref{fig:example_appendix}b}: This is an example where both $(x_{k_q},x_i)$ and $(x_{k_q},x_j)$ are invisible, and are bows. There are two visible pairs in this example: $x_4 \rightarrow x_3$ is a visible edge, and $(x_1,x_2)$ is a visible non-edge. Applying Lemma~\ref{lemma:visible_non_edge_implication} to $(x_1,x_2)$ with $K = \{x_3\}$, we get that $x_3$ is a parent of either $x_1$ or $x_2$. Note that both $(x_1,x_3)$ and $(x_2,x_3)$ are invisible, and are also bows. Applying Lemma~\ref{lemma:visible_edge_implication} to $x_4 \rightarrow x_3$ with $K = \{x_1\}$, we get that $x_1$ is an ancestor of $x_3$. Therefore, $x_3$ must be a parent of $x_2$. Thus, the parent–child relationship in the bow $(x_3,x_2)$ is fully resolved.

\begin{figure}[!ht]
\centering
\includegraphics[width = 0.75\columnwidth]{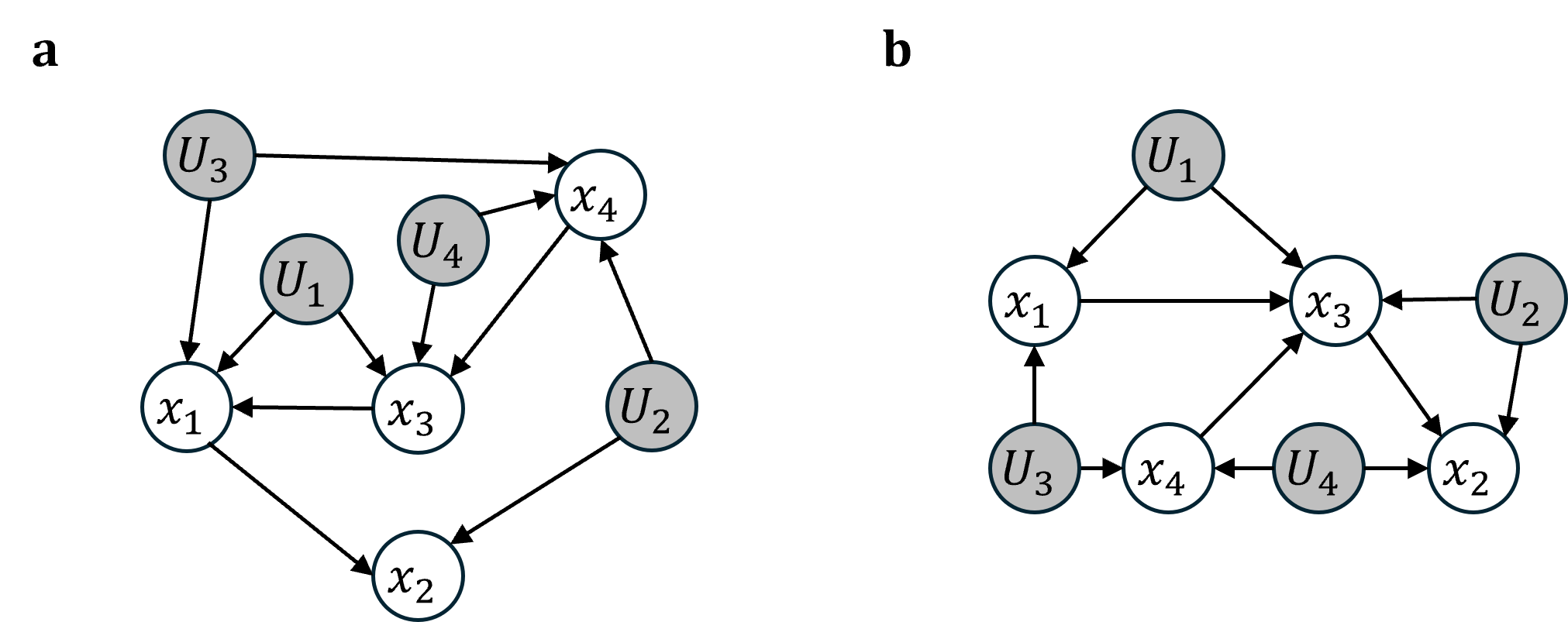}
\caption{Examples where the regression approach alone is sufficient to resolve parent–child relationships in an invisible pair.\label{fig:example_appendix}}
\end{figure}

\section{\MakeUppercase{An example where Corollary~\ref{coro:ancestorship} can identify the causal direction in an invisible pair}}\label{appendix:example_ancestor}

Consider Fig.~\ref{fig:ancestor_example_appendix}. The pair $(x_2,x_3)$ is invisible, due to the UCP $x_2 \rightarrow U_1 \rightarrow x_3$.

\begin{figure}[!ht]
\centering
\includegraphics[width = 0.25\columnwidth]{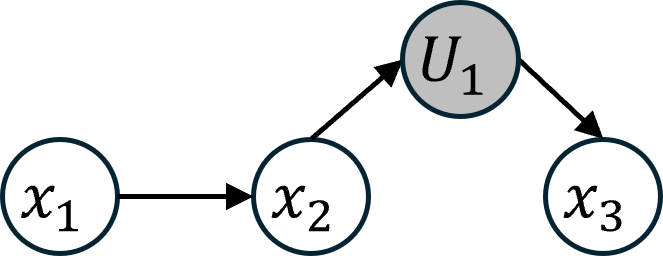}
\caption{Corollary~\ref{coro:ancestorship} can identify the ancestor relationship, i.e., causal direction, in the invisible pair $(x_2,x_3)$, whereas neither the regression approach alone nor conditional independence alone, e.g., the FCI framework, can.\label{fig:ancestor_example_appendix}}
\end{figure}

\textbf{Using Corollary~\ref{coro:ancestorship}}. $x_1 \rightarrow x_2$ is a visible edge, and is identifiable by Lemma~\ref{lemma:visible_parent}. Thus, $x_1$ can be identified as a parent of $x_2$, and therefore an ancestor of $x_2$. $(x_2,x_3)$ can be identified as an invisible pair by Lemma~\ref{lemma:invisible}. Furthermore, $x_1 \cind x_3 \mid x_2$ if we assume the causal Markov condition. Therefore, the condition of Corollary~\ref{coro:ancestorship} is satisfied, and thus $x_2$ can be identified as an ancestor of $x_3$ by Corollary~\ref{coro:ancestorship}.

\textbf{Using the regression approach alone.} $x_1 \rightarrow x_2$ can be identified as a visible edge. $(x_2,x_3)$ is identified as invisible, due to Lemma~\ref{lemma:invisible}. The pair $(x_1,x_3)$ is invisible, due to the UCP $x_1 \rightarrow x_2 \rightarrow U_1 \rightarrow x_3$. Thus, $(x_1,x_3)$ is identified as invisible by Lemma~\ref{lemma:invisible}. Lemmas~\ref{lemma:visible_edge_implication} and~\ref{lemma:visible_non_edge_implication} cannot be applied in this example. Therefore, the causal direction in $(x_2,x_3)$ is unidentifiable by the regression approach alone.

\textbf{Using conditional independence alone}. With the causal Markov condition and CFC, step~1 of FCI outputs the skeleton $x_1 \circ \mathdash \circ x_2 \circ \mathdash \circ x_3$, where there is one unshielded triple $x_1 \circ \mathdash \circ x_2 \circ \mathdash \circ x_3$, with $x_1 \cind x_3 \mid \{x_2\}$. Since the middle node $x_2$ of this unshielded triple belongs to the $d$-separation set $\{x_2\}$, one cannot identify more causal directions. Therefore, the final output of FCI is $x_1 \circ \mathdash \circ x_2 \circ \mathdash \circ x_3$, where the causal direction between $x_2$ and $x_3$ is unidentifiable.

\section{\MakeUppercase{Related works}}\label{sec:related_works}
CAMs~\citep{Buhlmann14CAM} belong to a subclass of the Additive Noise Models (ANMs)~\citep{Hoyer09NIPS}, which are causal models that assume nonlinear causal functions with additive noise terms, but the causal effects can be non-additive. Another important causal model is the Linear Non-Gaussian Acyclic Model (LiNGAM)~\citep{Shimizu06JMLR}, which assumes linear causal relationships with non-Gaussian noise.

A key extension of these causal models involves cases with hidden common causes~\citep{Hoyer08IJAR,Zhang10UAI-GP,Tashiro14NECO,Salehkaleybar2020learning}. To address such scenarios, methods like the Repetitive Causal Discovery (RCD) algorithm~\citep{Maeda20AISTATS} and the CAM-UV (Causal Additive Models with Unobserved Variables) algorithm~\citep{maeda21a} have been developed. CAM-UV has been applied to integrate causal graphs estimated from non-overlapping variable sets~\citep{Suzuki2026}.

Estimation approaches for causal discovery can be generally categorized into three groups: constraint-based methods~\citep{Spirtes91PC,Spirtes95FCI}, score-based methods~\citep{Chickering2002}, and continuous-optimization-based methods~\citep{Zheng18Neurips,Bhattacharya21ABC}. Additionally, a hybrid approach that combines the ideas of constraint-based and score-based methods has been proposed for non-parametric cases~\citep{Ogarrio16Hybrid}.

\section{\MakeUppercase{Running time of CAM-UV-X}}\label{appendix:running_time}
Assume that each regression, each calculation of p-HSIC, and each conditional independence check costs $O(n^2)$, where $n$ is the number of samples. After executing CAM-UV, the remaining running time of CAM-UV-X is bounded by the time of executing lines 4-6, where \texttt{checkVisible} is executed for every NaN entry of the adjacency matrix, and by the time of executing lines 8-9, where \texttt{checkOnPath} is executed for every non-NaN entry of the adjacency matrix. Let the two times be $t_1$ and $t_2$, respectively. 
\begin{itemize}
\item For $t_1$: The number of NaN entries in the adjacency matrix is $O(p^2)$. In each execution of \texttt{checkVisible}, in the worst case, one needs to check all sets $M \subseteq Q$ and $N \subseteq Q$, where the size of $Q$ is at most $p-2$. Therefore, $t_1$ is $O(p^2 \times 2^{p - 2} \times 2^{p - 2} \times n^2)$.
\item For $t_2$: The number of non-NaN entries of the adjacency matrix is $O(p^2)$. For each execution of \texttt{checkOnPath}, a third variable $x_k$ is considered. For each $x_k$, in the worst case, the procedure checks all sets $M \subseteq X\setminus \{x_i,x_k\}$ and $N \subseteq X\setminus \{x_j,x_k\}$. The size of $X\setminus \{x_i,x_k\}$ is $p-2$. Therefore, $t_2$ is $O(p^2 \times p \times 2^{p - 2} \times 2^{p - 2} \times n^2)$.
\end{itemize}
Given that the running time of CAM-UV is $O(p2^{p}n^2)$~\citep{maeda21a}, the running time of CAM-UV-X is $O(p2^{p}n^2 + p^22^{2p - 4}n^2 + p^32^{2p-4}n^2) = O(p^32^{2p - 4}n^2)$.   

\section{\MakeUppercase{Additional theoretical results}}\label{appendix:proof}
\subsection{Some additional identifiability results}\label{appendix:additional_results}

We provide another sufficient condition for certifying that some variable $x_k$ must be a parent of $x_i$ or $x_j$, even if $(x_k,x_i)$ and $(x_k,x_j)$ are both invisible, by utilizing visible non-edges. First, we have the following sufficient condition for non-parent relations and non-edges.
\begin{lemma}\label{lemma:non_parent_condition}
If the following equation holds:
\begin{align}
&\exists G_1,G_2\in \mathcal{G}, M \subseteq X \setminus \{x_i,x_j\}, N \subseteq X \setminus \{x_i,x_j\}: \nonumber \\
&x_i - G_1(M \cup \{x_j\}) \cind x_j - G_2(N), \label{eq:not_parent_1}
\end{align}
then $(x_i,x_j)$ is not invisible, and $x_i$ is not a parent of $x_j$.

Similarly, if the following equation holds:
\begin{align}
&\exists G_1,G_2\in \mathcal{G}, M \subseteq X \setminus \{x_i,x_j\}, N \subseteq X \setminus \{x_i,x_j\}: \nonumber \\
&x_i - G_1(M) \cind x_j - G_2(N\cup \{x_i\}), \label{eq:not_parent_2}
\end{align}
then $(x_i,x_j)$ is not invisible, and $x_j$ is not a parent of $x_i$. If Eqs.~(\ref{eq:not_parent_1}) and~(\ref{eq:not_parent_2}) are satisfied, $(x_i,x_j)$ is a visible non-edge.
\end{lemma}
Eq.~(\ref{eq:not_parent_1}) means that, if the residual of regressing $x_i$ on a set containing $x_j$ is independent of the residual of regressing $x_j$ on a set that does not contain $x_i$, the pair is not invisible, and $x_i$ is not a parent of $x_j$. When Eqs.~(\ref{eq:not_parent_1}) and~(\ref{eq:not_parent_2}) are satisfied, both directions of parent relations are forbidden; thus $(x_i,x_j)$ is a visible non-edge. Instead of Eq.~(\ref{eq:non_edge}), Eqs.~(\ref{eq:not_parent_1}) and~(\ref{eq:not_parent_2}) provide an alternative, and sometimes more convenient, way to certify visible non-edges.

Similar to Lemma~\ref{lemma:visible_non_edge_implication}, we can use the above condition to identify causal relationships in potentially invisible pairs as follows.
\begin{lemma}
\label{lemma:visible_non_edge_implication_2}
Consider distinct $x_i$, $x_j$, and $x_{k_1},\ldots,x_{k_m}$. Let $K = \{x_{k_1},\ldots,x_{k_m}\}$. If Eqs.~(\ref{eq:not_parent_1}),~(\ref{eq:not_parent_2}), and~(\ref{eq:visible_edge_implication_1}) are satisfied, $(x_j,x_i)$ is a visible non-edge, and each $x_{k_q}$ is either a parent of $x_i$ or a parent of $x_j$.
\end{lemma}
\subsection{Proof of Proposition~\ref{prop:invisibility_switch_meaning}}
If $(x_i,x_j)$ is an invisible pair with respect to $X \setminus \{x_{k}\}$, there exists a UBP/UCP between $(x_i,x_j)$ with respect to $X \setminus \{x_k\}$. Call this UBP/UCP $P$.  
If $(x_i,x_j)$ is visible with respect to $X$, there are no UBPs/UCPs between $(x_i,x_j)$ with respect to $X$. This implies that $P$ is a UBP/UCP between $(x_i,x_j)$ with respect to $X \setminus \{x_k \}$, and is not a UBP/UCP between $(x_i,x_j)$ with respect to $X$. Thus, $x_k$ must be a parent of $x_j$ or a parent of $x_i$. 

\subsection{Proof of Lemma~\ref{lemma:visible_edge_implication}}

($\Leftarrow$):
If Eq.~(\ref{eq:visible_edge_implication_1}) is satisfied, $(x_i,x_j)$ is an invisible pair with respect to $X \setminus \{x_{k_q}\}$ due to Lemma~\ref{lemma:invisible}.  
When Eqs.~(\ref{eq:visible_edge_implication_2}) and~(\ref{eq:visible_edge_implication_3}) are satisfied, $x_j$ is a visible parent of $x_i$ with respect to $X$, due to Lemma~\ref{lemma:visible_parent}. Thus, due to Proposition~\ref{prop:invisibility_switch_meaning}, each $x_{k_q}$ must be a parent of $x_j$ or a parent of $x_i$. 

($\Rightarrow$):
Since $(x_i,x_j)$ is invisible with respect to $X\setminus \{x_{k_q}\}$, Eq.~(\ref{eq:visible_edge_implication_1}) holds, due to Lemma~\ref{lemma:invisible}. Since $x_j$ is a visible parent of $x_i$, Eqs.~(\ref{eq:visible_edge_implication_2}) and
\begin{equation}
\exists G_1,G_2\in \mathcal{G},M \subseteq X \setminus \{x_i,x_j\}, N \subseteq X \setminus \{x_i,x_j\}: 
x_i - G_1(M\cup \{x_j\}) \cind x_j - G_2(N) \label{eq:temp_1}
\end{equation}
are satisfied, due to Proposition~\ref{proposition:1}. The remaining task is to show that Eq.~(\ref{eq:visible_edge_implication_3}) holds. To do so, one way is to show that, for any $x_{k_q} \in K$, either $x_{k_q} \in M$ or $x_{k_q} \in N$ with $M,N$ satisfying Eq.~(\ref{eq:temp_1}). Suppose that this is false, i.e., there exists some $x_{k_q}$ such that $x_{k_q} \notin M$ and $x_{k_q} \notin N$, for $M,N$ satisfying Eq.~(\ref{eq:temp_1}). Then, Eq.~(\ref{eq:temp_1}) implies that:
\begin{equation}
\exists G_1,G_2\in \mathcal{G},M \subseteq X \setminus \{x_i,x_{k_q}\}, N \subseteq X \setminus \{x_i,x_j,x_{k_q}\}: 
x_i - G_1(M) \cind x_j - G_2(N),
\end{equation}
which contradicts Eq.~(\ref{eq:visible_edge_implication_1}). Therefore, any $x_{k_q} \in K$ satisfies $x_{k_q} \in M$ or $x_{k_q} \in N$.

\subsection{Proof of Lemma~\ref{lemma:visible_non_edge_implication}}

($\Leftarrow$): If Eq.~(\ref{eq:visible_edge_implication_1}) is satisfied, $(x_i,x_j)$ is invisible with respect to $X \setminus \{x_{k_q}\}$, due to  Lemma~\ref{lemma:invisible}. When Eq.~(\ref{eq:visible_non_edge_implication_1}) is satisfied, $(x_i,x_j)$ is a visible non-edge, due to Lemma~\ref{lemma:visible_non_edge}. Due to Proposition~\ref{prop:invisibility_switch_meaning}, each $x_{k_q}$ must be a parent of $x_j$ or a parent of $x_i$.  

($\Rightarrow$): Since $(x_i,x_j)$ is invisible with respect to $X\setminus \{x_{k_q}\}$, Eq.~(\ref{eq:visible_edge_implication_1}) holds, due to Lemma~\ref{lemma:invisible}. Since $(x_i,x_j)$ is a visible non-edge, we have 
\begin{equation}
\exists G_1,G_2\in \mathcal{G},M \subseteq X \setminus \{x_i, x_j\}, N \subseteq X \setminus \{x_i,x_j\}:
x_i - G_1(M) \cind x_j - G_2(N),\label{eq:temp_2}
\end{equation}
due to Lemma~\ref{lemma:visible_non_edge}. 
To show Eq.~(\ref{eq:visible_non_edge_implication_1}), we will show that for every $x_{k_q} \in K$, either $x_{k_q} \in M$ or $x_{k_q} \in N$, with $M,N$ satisfying Eq.~(\ref{eq:temp_2}). Suppose that this is false, i.e., for any $M,N$ that satisfies Eq.~(\ref{eq:temp_2}), there exists some $x_{k_q}$ such that $x_{k_q} \notin M$ and $x_{k_q} \notin N$. Eq.~(\ref{eq:temp_2}) implies that:
\begin{equation}
\exists G_1,G_2\in \mathcal{G},M \subseteq X \setminus \{x_i, x_j,x_{k_q}\}, N \subseteq X \setminus \{x_i,x_j,x_{k_q}\}:
x_i - G_1(M) \cind x_j - G_2(N),
\end{equation}
which contradicts Eq.~(\ref{eq:visible_edge_implication_1}). Therefore, every $x_{k_q} \in K$ is in $M$ or in $N$.

\subsection{Proof of Lemma~\ref{lemma:non_parent_condition}}
Assume Eq.~(\ref{eq:not_parent_1}). We will prove that $(x_i,x_j)$ is not invisible, and $x_i$ is not a parent of $x_j$, by contradiction. 

If $(x_i,x_j)$ is invisible, Eq.~(\ref{eq:invisible}) is satisfied, which contradicts Eq.~(\ref{eq:not_parent_1}). Therefore, $(x_i,x_j)$ is not invisible. 

Suppose that $x_i$ is a parent of $x_j$. This means that $x_i$ is a visible parent of $x_j$. Eq.~(\ref{eq:visible_parent_1}) is satisfied, due to Lemma~\ref{lemma:visible_parent}. However, this contradicts Eq.~(\ref{eq:not_parent_1}). Therefore, $x_i$ is not a parent of $x_j$.

When both Eqs.~(\ref{eq:not_parent_1}) and~(\ref{eq:not_parent_2}) are satisfied, $(x_i,x_j)$ is not invisible, thus one of the following three cases must happen: 1) $x_i$ is a visible parent of $x_j$, 2) $x_j$ is a visible parent of $x_i$, or 3) $(x_i,x_j)$ is a visible non-edge. However, Eq.~(\ref{eq:not_parent_1}) implies $x_i$ is not a parent of $x_j$, and Eq.~(\ref{eq:not_parent_2}) implies that $x_j$ is not a parent of $x_i$. Therefore, $(x_i,x_j)$ must be a visible non-edge.

\subsection{Proof of Lemma~\ref{lemma:visible_non_edge_implication_2}}
Eq.~(\ref{eq:visible_edge_implication_1}) implies that $(x_i,x_j)$ is invisible with respect to $X \setminus \{x_{k_q}\}$, due to  Lemma~\ref{lemma:invisible}. Eqs.~(\ref{eq:not_parent_1}) and~(\ref{eq:not_parent_2}) imply that $(x_i,x_j)$ is a visible non-edge with respect to $X$, by Lemma~\ref{lemma:non_parent_condition}. Due to Proposition~\ref{prop:invisibility_switch_meaning}, each $x_{k_q}$ must be a parent of $x_i$ or $x_j$.

\subsection{Proof of Lemma~\ref{lemma:y_structure}}
As mentioned in the main text, the lemma is standard and not new. For completeness, we record its proof here. 

We prove by contradiction.
\begin{enumerate}
\item Suppose there is a backdoor path $x_i \leftarrow \cdots \leftarrow v \rightarrow \cdots \rightarrow x_j$. Since $x_k$ is an ancestor of $x_i$, the path $x_k \rightarrow \cdots \rightarrow x_i \leftarrow \cdots \leftarrow v \rightarrow \cdots \rightarrow x_j$ exists. By conditioning on $x_i$, which is a collider on this path, the path is open and thus $x_k$ and $x_j$ cannot be independent due to CFC. This contradicts the assumption $x_k \cind x_j \mid x_i$.
\item Suppose $x_j$ is an ancestor of $x_i$. Since $x_k$ is an ancestor of $x_i$, the path $x_k \rightarrow \cdots \rightarrow x_i \leftarrow \cdots \leftarrow x_j$ exists. By conditioning on $x_i$, which is a collider on this path, the path is open and thus $x_k$ and $x_j$ cannot be independent due to CFC. This contradicts the assumption $x_k \cind x_j \mid x_i$. 
\end{enumerate}

\subsection{Proof of Corollary~\ref{coro:ancestorship}}
Eq.~(\ref{eq:invisible}) implies that a UBP/UCP must exist between $x_i$ and $x_j$, due to Lemma~\ref{lemma:invisible}. Lemma~\ref{lemma:y_structure} rules out the possibilities of any UBP and the causal direction from $x_j$ to $x_i$. Therefore, the only possible scenario is that there is a UCP from $x_i$ to $x_j$, which means $x_i$ is an ancestor of $x_j$.

\subsection{Proof of Corollary~\ref{coro:parentship}}
Lemma~\ref{lemma:y_structure} implies that $x_k$ is not an ancestor, and thus not a parent, of $x_i$. Therefore, the only possibility is that $x_k$ is a parent of $x_j$.

\subsection{Proof of Theorem~\ref{theorem:CAM_UV_X_sound_complete}}\label{appendix:proof_theorem}

We state some facts that can limit the possible contents of the regression sets $M$ and $N$ in Lemmas~\ref{lemma:visible_parent} and~\ref{lemma:visible_non_edge}. 

Lemma~\ref{lemma:visible_parent} is equivalent to the following proposition.

\begin{proposition}~\label{proposition:1} Consider $X' \subseteq X$ and $x_i,x_j \in X'$. 
$x_j$ is a visible parent of $x_i$ with respect to $X'$ if and only if Eq.~(\ref{eq:visible_parent_1}) and 
\begin{equation}
\exists G_1,G_2\in \mathcal{G},M \subseteq X' \setminus \{x_i,x_j\}, N \subseteq X' \setminus \{x_i,x_j\}:
x_i - G_1(M\cup \{x_j\}) \cind x_j - G_2(N) \label{eq:visible_parent_2_dashed}
\end{equation}
are satisfied. Eq.~(\ref{eq:visible_parent_2_dashed}) means that, in Eq.~(\ref{eq:visible_parent_2}), one must regress $x_i$ on a set that contains $x_j$.
\end{proposition}
\begin{proof}
If Eqs.~(\ref{eq:visible_parent_1}) and~(\ref{eq:visible_parent_2_dashed}) are satisfied, Eqs.~(\ref{eq:visible_parent_1}) and~(\ref{eq:visible_parent_2}) are satisfied. Thus, $x_j$ is a visible parent of $x_i$ with respect to $X'$ by Lemma~\ref{lemma:visible_parent}.

Suppose that $x_j$ is a visible parent of $x_i$ with respect to $X'$. Eqs.~(\ref{eq:visible_parent_1}) and~(\ref{eq:visible_parent_2}) are satisfied by Lemma~\ref{lemma:visible_parent}. From Eq.~(\ref{eq:visible_parent_2}), we have:
\begin{equation}
\exists G'_1,G'_2\in \mathcal{G},M' \subseteq X' \setminus \{x_i\}, N' \subseteq X' \setminus \{x_i,x_j\}:
x_i - G'_1(M') \cind x_j - G'_2(N').
\end{equation}
For any $M'\subseteq X' \setminus \{x_i\}$ satisfying the above equation, we prove that $x_j$ must belong to $M'$ by contradiction. Suppose that $x_j \notin M'$. Then $M' \subseteq X'\setminus \{x_i,x_j\}$, which means that
\begin{equation}
\exists G'_1,G'_2\in \mathcal{G},M' \subseteq X' \setminus \{x_i,x_j\}, N' \subseteq X' \setminus \{x_i,x_j\}:
x_i - G'_1(M') \cind x_j - G'_2(N').
\end{equation}
By Lemma~\ref{lemma:visible_non_edge}, this implies that $(x_i,x_j)$ is a visible non-edge with respect to $X'$. However, this contradicts the fact that $x_j$ is a parent of $x_i$. Therefore, $x_j \in M'$, which means Eq.~(\ref{eq:visible_parent_2_dashed}) is true.
\end{proof}
We have the following proposition to limit the content of the regression set $M$ in Eq.~(\ref{eq:visible_parent_2_dashed}).
\begin{proposition}~\label{proposition:2}
Assume that $x_k$ is not a parent of $x_i$ and not a parent of $x_j$.  The following equation
\begin{equation}
\exists G_1,G_2\in \mathcal{G}, M \subseteq X \setminus \{x_i,x_k\}, N \subseteq X \setminus \{x_i,x_j\}: x_i - G_1(M\cup \{x_k\}) \cind x_j - G_2(N) \label{eq:visible_parent_2_a}
\end{equation}
implies that
\begin{equation}
\exists G_1,G_2\in \mathcal{G}, M \subseteq X \setminus \{x_i,x_k\}, N \subseteq X \setminus \{x_i,x_j\}: x_i - G_1(M) \cind x_j - G_2(N). \label{eq:visible_parent_2_b}
\end{equation}
In other words, one can limit the set $M$ to the union of the set of parents of $x_i$ and the set of parents of $x_j$ when checking Eq.~(\ref{eq:visible_parent_2_dashed}).
\end{proposition}
\begin{proof}
Eq.~(\ref{eq:visible_parent_2_a}) implies that $(x_i,x_j)$ is visible with respect to $X$. Since $x_k$ is not a parent of $x_i$ or a parent of $x_j$, if there are no UBPs/UCPs between $(x_i,x_j)$ with respect to $X$, then there are also no UBPs/UCPs between $(x_i,x_j)$ with respect to $X \setminus \{x_k\}$. Therefore, $(x_i,x_j)$ is also visible with respect to $X\setminus \{x_k\}$, and we have
\begin{equation}
\exists G_1,G_2\in \mathcal{G}, M \subseteq X \setminus \{x_i,x_k\}, N \subseteq X \setminus \{x_i,x_j,x_k\}: x_i - G_1(M) \cind x_j - G_2(N),
\end{equation}
which implies Eq.~(\ref{eq:visible_parent_2_b}).
\end{proof}
Similarly, we have the following proposition.
\begin{proposition}~\label{proposition:3}
Assume that $x_k$ is not a parent of $x_i$ and not a parent of $x_j$. The following equation
\begin{equation}
\exists G_1,G_2\in \mathcal{G}, M \subseteq X \setminus \{x_i\}, N \subseteq X \setminus \{x_i,x_j,x_k\}: x_i - G_1(M) \cind x_j - G_2(N\cup \{x_k\})\label{eq:visible_parent_2_c}
\end{equation}
implies that
\begin{equation}
\exists G_1,G_2\in \mathcal{G}, M \subseteq X \setminus \{x_i\}, N \subseteq X \setminus \{x_i,x_j,x_k\}: x_i - G_1(M) \cind x_j - G_2(N). \label{eq:visible_parent_2_d}
\end{equation}
In other words, one can also limit the set $N$ to the union of the set of parents of $x_i$ and the set of parents of $x_j$ when checking Eq.~(\ref{eq:visible_parent_2_dashed}).
\end{proposition}
\begin{proof}
Same as the proof of Proposition~\ref{proposition:2}.
\end{proof}

We have the following Propositions~\ref{proposition:4} and~\ref{proposition:5} to limit the contents of the regression sets $M$ and $N$ when checking Eq.~(\ref{eq:non_edge}) in Lemma~\ref{lemma:visible_non_edge}. 

\begin{proposition}~\label{proposition:4}
Assume that $x_k$ is not a parent of $x_i$ and not a parent of $x_j$.  The following equation
\begin{equation}
\exists G_1,G_2\in \mathcal{G}, M \subseteq X \setminus \{x_i,x_j,x_k\}, N \subseteq X \setminus \{x_i,x_j\}: x_i - G_1(M\cup \{x_k\}) \cind x_j - G_2(N) \label{eq:visible_non_edge_a}
\end{equation}
implies that
\begin{equation}
\exists G_1,G_2\in \mathcal{G}, M \subseteq X \setminus \{x_i,x_j,x_k\}, N \subseteq X \setminus \{x_i,x_j\}: x_i - G_1(M) \cind x_j - G_2(N). \label{eq:visible_non_edge_b}
\end{equation}
\end{proposition}
\begin{proof}
Eq.~(\ref{eq:visible_non_edge_a}) implies that $(x_i,x_j)$ is a visible non-edge with respect to $X$. Since $x_k$ is not a parent of $x_i$ or a parent of $x_j$, if there are no UBPs/UCPs between $(x_i,x_j)$ with respect to $X$, then there are also no UBPs/UCPs between $(x_i,x_j)$ with respect to $X \setminus \{x_k\}$. Therefore, $(x_i,x_j)$ is also a visible non-edge with respect to $X\setminus \{x_k\}$, and we have
\begin{equation}
\exists G_1,G_2\in \mathcal{G}, M \subseteq X \setminus \{x_i,x_j,x_k\}, N \subseteq X \setminus \{x_i,x_j,x_k\}: x_i - G_1(M) \cind x_j - G_2(N),
\end{equation}
which implies Eq.~(\ref{eq:visible_non_edge_b}).
\end{proof}

\begin{proposition}~\label{proposition:5}
Assume that $x_k$ is not a parent of $x_i$ and not a parent of $x_j$.  The following equation
\begin{equation}
\exists G_1,G_2\in \mathcal{G}, M \subseteq X \setminus \{x_i,x_j\}, N \subseteq X \setminus \{x_i,x_j,x_k\}: x_i - G_1(M) \cind x_j - G_2(N\cup \{x_k\}) \label{eq:visible_non_edge_c}
\end{equation}
implies that
\begin{equation}
\exists G_1,G_2\in \mathcal{G}, M \subseteq X \setminus \{x_i,x_j\}, N \subseteq X \setminus \{x_i,x_j,x_k\}: x_i - G_1(M) \cind x_j - G_2(N). \label{eq:visible_non_edge_d}
\end{equation}
\end{proposition}
\begin{proof}
Same as the proof of Proposition~\ref{proposition:4}.
\end{proof}

Now we are ready to prove the soundness and completeness of CAM-UV-X as follows.

\subsubsection{Soundness in identifying visible non-edges}
A pair $(x_i, x_j)$ is identified as a visible non-edge in CAM-UV-X if 1) it is identified as a visible non-edge by CAM-UV, or 2) it is identified as invisible by CAM-UV, and is re-identified as a visible non-edge by \texttt{checkVisible}. 
\begin{itemize}
\item For the first case, since CAM-UV is sound in identifying visible pairs with Assumption~\ref{assumption:correct_regression}, $(x_i,x_j)$ is also a visible non-edge in the ground truth. 
\item For the second case, this means that either 1) $e >\alpha$ in line 6, or 2) $iNotParent = True$ and $jNotParent = True$ in line 17 of \texttt{checkVisible}. 
\begin{itemize}

\item If $e >\alpha$, Eq.~(\ref{eq:non_edge}) with $X' = X$ is satisfied due to Assumption~\ref{assumption:correct_regression}. This implies $(x_i,x_j)$ is a visible non-edge in the ground truth, due to Lemma~\ref{lemma:visible_non_edge}. 
\item If $iNotParent = True$ and $jNotParent = True$, Eqs.~(\ref{eq:not_parent_1}) and~(\ref{eq:not_parent_2}) are satisfied, due to Assumption~\ref{assumption:correct_regression}. This implies that $(x_i,x_j)$ is a visible non-edge in the ground truth, due to Lemma~\ref{lemma:non_parent_condition}.
\end{itemize}
\end{itemize}
Therefore, CAM-UV-X is sound in identifying visible non-edges.

\subsubsection{Completeness in identifying visible non-edges.}
Suppose that $(x_i,x_j)$ is a visible non-edge in the ground truth. Due to Lemma~\ref{lemma:visible_non_edge}
, Eq.~(\ref{eq:non_edge}) with $X' = X$ holds in the data. After the execution of CAM-UV, there are three cases:
\begin{itemize}
\item $(x_i,x_j)$ is concluded as a visible non-edge: CAM-UV-X leaves the pair as is. Therefore, $(x_i,x_j)$ is also a visible non-edge in the output of CAM-UV-X.
\item $(x_i,x_j)$ is concluded as a visible edge: this case does not happen, since CAM-UV is sound in identifying visible pairs with Assumption~\ref{assumption:correct_regression}.
\item $(x_i,x_j)$ is concluded as an invisible pair: \texttt{checkVisible} is executed for $(i,j)$. In lines 5 and 6 of \texttt{checkVisible}, independence is checked for all subsets of $Q$, which is the union of $P_i$, $P_j$, and the set of nodes whose parent–child relationship to either $x_i$ or $x_j$ is not clear. Since CAM-UV is sound for identifying visible non-edges with Assumption~\ref{assumption:correct_regression}, this set does not miss any parents of $x_i$ or $x_j$.  Since Eq.~(\ref{eq:non_edge}) holds with $X' = X$, there exist some sets $M\subseteq X \setminus \{x_i,x_j\}$ and $N\subseteq X\setminus \{x_i,x_j\}$ that realize the independence between residuals, and produce a p-HSIC value that is greater than $\alpha$, due to Assumption~\ref{assumption:correct_regression}. 
\begin{itemize}
\item If $M$ and $N$ do not contain non-parents of $x_i$ and $x_j$, this means $M,N \subseteq Q$, and exhaustively searching through $Q$ as in \texttt{checkVisible} ensures the finding of $M$ and $N$.

\item If $M$ or $N$ contain some non-parent of $x_i$ and $x_j$, Propositions~\ref{proposition:4} and~\ref{proposition:5} imply that there exist some sets $M'$ and $N'$ in $Q$ such that the independence between residuals is realized, and produce a p-HSIC value that is greater than $\alpha$, by Assumption~\ref{assumption:correct_regression}. Exhaustively searching through $Q$ as in \texttt{checkVisible} ensures the finding of $M'$ and $N'$. 
\end{itemize}
Therefore, \texttt{checkVisible} is guaranteed to find some set $M,N\subseteq Q$ such that $e > \alpha$ and line 6 is satisfied. Thus, CAM-UV-X outputs $(x_i,x_j)$ as a visible non-edge.  
\end{itemize}
Therefore, CAM-UV-X is complete in identifying visible non-edges.

\subsubsection{Soundness in identifying visible edges}
An edge $x_j\rightarrow x_i$ is identified as visible in CAM-UV-X if 1) it is identified as visible by CAM-UV, or 2) it is identified as invisible by CAM-UV, and is re-identified as visible by \texttt{checkVisible}. 
\begin{itemize}
\item For the first case, since CAM-UV is sound in identifying visible pairs with Assumption~\ref{assumption:correct_regression}, $x_j\rightarrow x_i$ is also a visible edge in the ground truth. 
\item For the second case, this means that $e \le \alpha$ in line 6 for all sets $M,N \subseteq Q$, $a_1 > \alpha$ (since $iNotParent$ must be $True$) for some set $M,N \subseteq Q$, and $a_2 \le \alpha$ (since $jNotParent$ must remain $False$) for all sets $M,N \subseteq Q$ in \texttt{checkVisible}. 
\begin{itemize}
\item  Since $a_1 > \alpha$ for some set $M,N \subseteq Q$, Eq.~(\ref{eq:visible_parent_2_dashed}) with $X' = Q \cup \{x_i,x_j\}$ is satisfied, due to Assumption~\ref{assumption:correct_regression}.

\item Since $e \le \alpha$ and $a_2 \le \alpha$ for all sets checked, this means that for all $M \subseteq Q$ and $N \subseteq Q \cup \{x_i\}$, independence is not established for all functions in $\mathcal{G}$, due to Assumption~\ref{assumption:correct_regression}. $Q$ is the union of $P_i$, $P_j$, and the set of nodes whose parent–child relationship to either $x_i$ or $x_j$ is not clear. Since CAM-UV is sound for identifying visible non-edges with Assumption~\ref{assumption:correct_regression}, this set does not miss any parents of $x_i$ or $x_j$. This implies that Eq.~(\ref{eq:visible_parent_1}) with $X' = Q \cup \{x_i,x_j\}$ is satisfied.
\end{itemize}
Since Eqs.~(\ref{eq:visible_parent_2_dashed}) and~(\ref{eq:visible_parent_1}) are satisfied with $X' = Q \cup \{x_i,x_j\}$, $x_j$ is a visible parent of $x_i$ with respect to $Q \cup \{x_i,x_j\}$ in the ground truth, due to Proposition~\ref{proposition:1}. This implies that $x_j$ is also a visible parent of $x_i$ with respect to $X$ in the ground truth.
\end{itemize}
Therefore, CAM-UV-X is sound in identifying visible edges.

\subsubsection{Completeness in identifying visible edges}
Suppose that $x_j\rightarrow x_i$ is a visible edge in the ground truth. Due to Proposition~\ref{proposition:1}, Eqs.~(\ref{eq:visible_parent_1}) and~(\ref{eq:visible_parent_2_dashed}) with $X' = X$ hold in the data. After the execution of CAM-UV, there are three cases:
\begin{itemize}
\item $x_j\rightarrow x_i$ is concluded as a visible edge: CAM-UV-X leaves the edge as is. Therefore, $x_j\rightarrow x_i$ is also a visible edge in the output of CAM-UV-X.
\item $x_j\rightarrow x_i$ is concluded as a visible non-edge: this case does not happen, since CAM-UV is sound in identifying visible pairs with Assumption~\ref{assumption:correct_regression}.
\item $x_j\rightarrow x_i$ is concluded as an invisible pair: \texttt{checkVisible} is executed for $(i,j)$. Due to Eq.~(\ref{eq:visible_parent_1}) with $X' = X$ and Assumption~\ref{assumption:correct_regression}, $e \le \alpha$ and $a_2 \le \alpha$ for all sets $M,N\subseteq Q$. Thus, $jNotParent$ remains $False$, and lines 7 and 18 of \texttt{checkVisible} are guaranteed to be not executed. $Q$ is the union of $P_i$, $P_j$, and the set of nodes whose parent–child relationship to either $x_i$ or $x_j$ is not clear. Since CAM-UV is sound for identifying visible non-edges with Assumption~\ref{assumption:correct_regression}, this set does not miss any parents of $x_i$ or $x_j$. Since Eq.~(\ref{eq:visible_parent_2_dashed}) with $X' = X$ holds, there exist some sets $M$ and $N$ that realize the independence between residuals. 
\begin{itemize}
\item If $M$ and $N$ do not contain non-parents of $x_i$ and $x_j$, this means $M,N \subseteq Q$, and searching through $Q$ ensures the finding of $M$ and $N$. 
\item If $M$ or $N$ contain some non-parent of $x_i$ and $x_j$, Propositions~\ref{proposition:2} and~\ref{proposition:3} imply that there exist some set $M'$ and $N'$ in $Q$ such that the independence between residuals is realized. Searching through $Q$ ensures the finding of $M'$ and $N'$.
\end{itemize}
Therefore, CAM-UV-X is guaranteed to find some sets such that regression in those sets will produce $a_1 > \alpha$ in line 11 of \texttt{checkVisible}, due to Assumption~\ref{assumption:correct_regression}. Therefore, $iNotParent$ will be changed to $True$. Coupling this with the fact that $jNotParent$ remains $False$, and the fact that lines 7 and 18 of \texttt{checkVisible} are not executed, one can conclude that line 23 of \texttt{checkVisible} is guaranteed to be executed. Thus, CAM-UV-X outputs $x_j\rightarrow x_i$ as a visible edge.  
\end{itemize}
Therefore, CAM-UV-X is complete in identifying visible edges.

\subsubsection{Soundness in identifying invisible pairs.} When CAM-UV-X identifies a pair $(x_i,x_j)$ as invisible, the pair cannot be a visible edge in the ground truth, since this would contradict the proven completeness of CAM-UV-X in identifying visible edges. It also cannot be a visible non-edge in the ground truth, since this would contradict the proven completeness of CAM-UV-X in identifying visible non-edges. Therefore, the pair must be invisible in the ground truth. This means CAM-UV-X is sound in identifying invisible pairs.
\subsubsection{Completeness in identifying invisible pairs.} 
Suppose $(x_i,x_j)$ is an invisible pair in the ground truth. In the output of CAM-UV-X, the pair cannot be a visible edge, since this would contradict the proven soundness of CAM-UV-X in identifying visible edges. It also cannot be a visible non-edge in the output, since this would contradict the proven soundness of CAM-UV-X in identifying visible non-edges. Therefore, the pair must be invisible in the output of CAM-UV-X. This means CAM-UV-X is complete in identifying invisible pairs.

\section{\MakeUppercase{Step-by-step Execution of CAM-UV on Figs.~\ref{fig:illustrative_1}a and~\ref{fig:illustrative_1}b}}\label{appendix:cam_uv_step_by_step}
We work out step-by-step the execution of the CAM-UV algorithm for the graphs in Figs.~\ref{fig:illustrative_1}a and~\ref{fig:illustrative_1}b. We assume Assumption~\ref{assumption:correct_regression}.  
\subsection{Fig.~\ref{fig:illustrative_1}a}
\begin{itemize}
\item Algorithm~1:
\begin{itemize}
\item Phase 1:
\begin{itemize}
\item $t = 2$: The sets $\{x_1,x_2\}$, $\{x_1,x_3\}$, and $\{x_2,x_3\}$ are considered as $K$. We search for the candidate sink $x_b$ in $K$. 
\begin{itemize}
\item $K = \{x_1,x_3\}$ or $K = \{x_2,x_3\}$: Since these pairs are invisible,  $x_b-G_1(M_b\cup K\setminus \{x_b\})$ and $x_j-G_2(M_j)$ for $j\in K\setminus \{x_b\}$ are not independent, regardless of $x_b$. Thus, line 15 will fail since $e \le \alpha$. There is no change in $M_i$.
\item $K = \{x_1,x_2\}$: CAM-UV correctly finds $x_b = x_2$. However, $x_b-G_1(M_b\cup K\setminus \{x_b\}) = x_2 - G_1(x_1)$ and $x_j-G_2(M_j) = x_1$ is not independent, due to the unblocked backdoor path $x_1 \leftarrow U_1 \rightarrow x_3 \rightarrow x_2$. Thus, line 15 will fail since $e \le \alpha$. There is no change in $M_i$.
\end{itemize}
$M_i$ remains empty for each $i$ and $t$ increases to $3$.
\item $t = 3$: $K = \{x_1,x_2,x_3\}$. The algorithm correctly chooses $x_b = x_2$. However, $x_b-G_1(M_b\cup K\setminus \{x_b\}) = x_2 - G(x_1,x_3)$ and $x_j-G_2(M_j) = x_3$ is not independent, due to the unobserved backdoor path $x_3 \leftarrow U_2 \rightarrow x_2$. Therefore, line 15 will fail again, since $e \le \alpha$.
\end{itemize}
$M_i$ remains empty for each $i$. Phase 1 of the Algorithm~1 ends.
\item Phase 2: Since $M_i$ is empty for each $i$, Phase~2 ends. 
\end{itemize}
Algorithm~1 ends with every $M_i$ being empty.
\item Algorithm~2: For each pair $(i,j)$, line 5 is satisfied. Therefore, the algorithm concludes that every pair is invisible. CAM-UV ends.
\end{itemize}

The final output is an adjacency matrix where each off-diagonal element is NaN.

\subsection{Fig.~\ref{fig:illustrative_1}b}
\begin{itemize}
\item Algorithm~1:
\begin{itemize}
\item Phase 1:
\begin{itemize}
\item $t = 2$: The sets $\{x_1,x_2\}$, $\{x_1,x_3\}$, $\{x_1,x_4\}$, $\{x_2,x_3\}$, $\{x_2,x_4\}$, and $\{x_3,x_4\}$ are considered as $K$. The candidate sink $x_b$ of $K$ is searched. 
\begin{itemize}
\item $K = \{x_1,x_2\}$: Regardless of which $x_b$ is, $x_b-G_1(M_b\cup K\setminus \{x_b\}) = x_b-G_1(\{x_1,x_2\}\setminus \{x_b\})$ and $x_j-G_2(M_j) = x_j$ for $j\in K\setminus \{x_b\}$ are not independent, since there are unblocked BPs/CPs when $x_3$ and $x_4$ are not added to the regression. Thus, line 15 will fail since $e \le \alpha$. There is no change in $M_i$.
\item The remaining pairs are all invisible. Therefore, $x_b-G_1(M_b\cup K\setminus \{x_b\}) = x_b - G_1(K\setminus \{x_b\})$ and $x_j-G_2(M_j)$ is not independent. Thus, line 15 will fail since $e \le \alpha$. There is no change in $M_i$.
\end{itemize}
$M_i$ remains empty for each $i$ and $t$ increases to $3$.
\item $t = 3$: 
\begin{itemize}
\item $K = \{x_1,x_2,x_3\}$: 
\begin{itemize}
\item $x_b = x_2$: $x_b-G_1(M_b\cup K\setminus \{x_b\}) = x_2 - G(x_1,x_3)$ and $x_j-G_2(M_j) = x_3$ is not independent, due to the unobserved backdoor path $x_3 \leftarrow U_4 \rightarrow x_2$. Therefore, line 15 will fail, since $e \le \alpha$. 
\item $x_b = x_1$: $x_b-G_1(M_b\cup K\setminus \{x_b\}) = x_1 - G(x_2,x_3)$ and $x_j-G_2(M_j) = x_3$ is not independent, due to the unobserved backdoor path $x_1 \leftarrow U_1 \rightarrow x_3$. Therefore, line 15 will fail, since $e \le \alpha$. 
\item $x_b = x_3$: $x_b-G_1(M_b\cup K\setminus \{x_b\}) = x_3 - G(x_1,x_2)$ and $x_j-G_2(M_j) = x_2$ is not independent, due to the unobserved backdoor path $x_3 \leftarrow U_4 \rightarrow x_2$. Therefore, line 15 will fail, since $e \le \alpha$. 
\end{itemize}

\item $K = \{x_2,x_3,x_4\}$: 
\begin{itemize}
\item $x_b = x_2$: $x_b-G_1(M_b\cup K\setminus \{x_b\}) = x_2 - G(x_3,x_4)$ and $x_j-G_2(M_j) = x_3$ is not independent, due to the unobserved backdoor path $x_3 \leftarrow U_4 \rightarrow x_2$. Therefore, line 15 will fail, since $e \le \alpha$. 
\item $x_b = x_3$: $x_b-G_1(M_b\cup K\setminus \{x_b\}) = x_3 - G(x_2,x_4)$ and $x_j-G_2(M_j) = x_2$ is not independent, due to the unobserved backdoor path $x_3 \leftarrow U_4 \rightarrow x_2$. Therefore, line 15 will fail, since $e \le \alpha$. 
\item $x_b = x_4$: $x_b-G_1(M_b\cup K\setminus \{x_b\}) = x_4 - G(x_2,x_3)$ and $x_j-G_2(M_j) = x_2$ is not independent, due to the unobserved backdoor path $x_2 \leftarrow U_3 \rightarrow x_4$. Therefore, line 15 will fail, since $e \le \alpha$. 
\end{itemize}
\item $K = \{x_1,x_3,x_4\}$:
\begin{itemize}
\item $x_b = x_1$: $x_b-G_1(M_b\cup K\setminus \{x_b\}) = x_1 - G(x_3,x_4)$ and $x_j-G_2(M_j) = x_3$ is not independent, due to the unobserved backdoor path $x_3 \leftarrow U_1 \rightarrow x_1$. Therefore, line 15 will fail, since $e \le \alpha$. 

\item $x_b = x_3$: $x_b-G_1(M_b\cup K\setminus \{x_b\}) = x_3 - G(x_1,x_4)$ and $x_j-G_2(M_j) = x_1$ is not independent, due to the unobserved backdoor path $x_3 \leftarrow U_1 \rightarrow x_1$. Therefore, line 15 will fail, since $e \le \alpha$. 
\item $x_b = x_4$: $x_b-G_1(M_b\cup K\setminus \{x_b\}) = x_4 - G(x_1,x_3)$ and $x_j-G_2(M_j) = x_1$ is not independent, due to the unobserved causal path $x_1 \rightarrow U_2 \rightarrow x_4$. Therefore, line 15 will fail, since $e \le \alpha$. 
\end{itemize}
\end{itemize}
$M_i$ remains empty for each $i$ and $t$ increases to $4$.

\item $t = 4$:
\begin{itemize}
\item $K = \{x_1,x_2,x_3,x_4\}$.

\begin{itemize}

\item $x_b = x_1$: $x_b-G_1(M_b\cup K\setminus \{x_b\}) = x_1 - G(x_2,x_3,x_4)$ and $x_j-G_2(M_j) = x_3$ is not independent, due to the unobserved backdoor path $x_3 \leftarrow U_1 \rightarrow x_1$. Therefore, line 15 will fail, since $e \le \alpha$. 

\item $x_b = x_2$: $x_b-G_1(M_b\cup K\setminus \{x_b\}) = x_2 - G(x_1,x_3,x_4)$ and $x_j-G_2(M_j) = x_3$ is not independent, due to the unobserved backdoor path $x_3 \leftarrow U_4 \rightarrow x_2$. Therefore, line 15 will fail, since $e \le \alpha$.

\item $x_b = x_3$: $x_b-G_1(M_b\cup K\setminus \{x_b\}) = x_3 - G(x_1,x_2,x_4)$ and $x_j-G_2(M_j) = x_1$ is not independent, due to the unobserved backdoor path $x_1 \leftarrow U_1 \rightarrow x_3$. Therefore, line 15 will fail, since $e \le \alpha$.

\item $x_b = x_4$: $x_b-G_1(M_b\cup K\setminus \{x_b\}) = x_4 - G(x_1,x_2,x_3)$ and $x_j-G_2(M_j) = x_2$ is not independent, due to the unobserved backdoor path $x_4 \leftarrow U_3 \rightarrow x_2$. Therefore, line 15 will fail, since $e \le \alpha$.

\end{itemize}
\end{itemize}
\end{itemize}

$M_i$ remains empty for each $i$. Phase 1 of the Algorithm~1 ends.
\item Phase 2: Since $M_i$ is empty for each $i$, Phase~2 ends. 
\end{itemize}
Algorithm~1 ends with every $M_i$ being empty.
\item Algorithm~2: For each pair $(i,j)$, line 5 is satisfied. Therefore, the algorithm concludes that every pair is invisible. CAM-UV ends.
\end{itemize}

The final output is an adjacency matrix where each off-diagonal element is NaN.

\section{\MakeUppercase{Details of experiments}}\label{appendix:experiment_details}

\subsection{S-B method}
As discussed in the main text, the theory of~\cite{Schultheiss_2024} is not designed for causal search. Below is a description of the S-B method, which is our attempt to create a baseline for identifying the parent–child relationships between observed variables using~\cite{Schultheiss_2024}.

\begin{itemize}
\item Initialize the adjacency matrix $A$ with all elements as NaN.
\item For each $j = 1,\ldots,p$:
\begin{itemize}
\item For a given significance level $\alpha$, apply Algorithm~3 of~\cite{Schultheiss_2024} with $\Tilde{\alpha} = \alpha$ and number of splits $B = 25$ to obtain $\hat{W}_j \subseteq \{1,\ldots,p\}$, the set of indices $i$ of variables $X_i$ whose causal effects on $X_j$ are well-specified. 
\item For each variable $X_i$ with $i \in \hat{W}_j$, we estimate the causal effect of $X_i$ on $X_j$ while all remaining variables in Eq.~(\ref{eq:CAM_UV}), including the hidden ones, are fixed. If this causal effect is $0$, then $X_i$ is not a parent of $X_j$. Otherwise, $X_i$ is a parent of $X_j$. The procedure is as follows.
\begin{itemize}
\item We learn the conditional expectation function $\mathbb{E}[X_{j}\mid X_{-j}]$ by learning a regression model $f(X_{-j})$ to predict $X_j$ using $X_{-j}$. Here, $X_{-j}$ are the observed variables excluding $X_j$. In the experiments, we use \mbox{XGBoost}~\citep{xgboost} as the regression model.  
\item We sample a location $m$ from $1,\ldots, n$, with $n$ being the number of observed samples, independently $M$ times. We use $M = 100$ in the experiments. For each location $m$, we calculate a value $\delta_m$ as follows.
\begin{itemize} 
\item For each location $m$, we randomly sample a location $a$ ($a \ne m$) from $1,\ldots,n$.
\begin{itemize}
\item We calculate $$\delta_m = \hat{f}\left(X_{i} = (x_{i})_a, X_{k|k\ne i,j} = (x_{k|k\ne i,j})_m\right) - \hat{f}\left(X_{i} = (x_{i})_m, X_{k|k\ne i,j} = (x_{k|k\ne i,j})_m\right).$$ $(x_{i})_a$ and $(x_{i})_m$ are the values of $X_i$ at the $a$-th and $m$-th samples, respectively. $(x_{k|k\ne i,j})_m$ are the values of the variables $X_{k|k\ne i,j}$ in the $m$-th sample.

\item $\delta_m$ is our estimate of the difference $$\mathbb{E}[X_j \mid X_{i} = (x_{i})_a, X_{k|k\ne i,j} = (x_{k|k\ne i,j})_m] - \mathbb{E}[X_j \mid X_{i} = (x_{i})_m, X_{k|k\ne i,j} = (x_{k|k\ne i,j})_m].$$
Due to Theorem~1 of~\cite{Schultheiss_2024}, this difference in conditional expectation values is the causal effect on the target variable $X_j$ when $X_i$ is changed from the value $(x_{i})_m$ to the value $(x_{i})_a$, while all remaining variables in Eq.~(\ref{eq:CAM_UV}), including the hidden ones, are fixed at their corresponding values in the $m$-th sample.

\item If $\delta_m \ne 0$, $X_i$ is a parent of $X_j$. Otherwise, $X_i$ is not a parent of $X_j$.

\end{itemize}
\end{itemize}
\item We use a $t$-test with significance level $\alpha$ for testing the null hypothesis of zero mean of the distribution of $\delta_m (m = 1,\ldots,M)$.
\begin{itemize}
\item If the null hypothesis is rejected, we conclude that $X_i$ is a parent of $X_j$ and set $A(j,i) = 1$ and $A(i,j) = 0$. 
\item If the null hypothesis is not rejected, we conclude that $X_i$ is not a parent of $X_j$, and set $A(j,i) = 0$.
\end{itemize}
\end{itemize}
\end{itemize}
\end{itemize}
\subsection{Computational environment}
All experiments were performed on a desktop computer with an Intel Core i7-13700K processor and 64 GB of RAM. 
\subsection{Metrics}
We define the metrics used in the experiments. $TP$ is the number of true positives. $TN$ is the number of true negatives. $FN$ is the number of false negatives. $FP$ is the number of false positives. 

In identifying the adjacency matrix $A$, for the case where the estimated $A(i,j)$ is NaN, we add $0.5\times P/(P+N)$, $0.5\times N/(P+N)$, $0.5\times P/(P+N)$, and $0.5\times N/(P+N)$ to $TP$, $TN$, $FN$, and $FP$, respectively, to reflect the expected values of a random guess. Here, $P$ and $N$ are the total number of positives and negatives in the ground truth, respectively.

The $precision$, $recall$, and $F_1$ are calculated as follows. $Precision$ is $TP/(TP + FP)$. $Recall$ is $TP/(TP + FN)$. $F_1$ is $2 * (precision * recall) / (precision + recall)$.

\subsection{ER graphs with Gaussian noise}\label{appendix:ER_results}
We generated random graphs from the ER model with 10 observed variables and edge probability $0.2$. We randomly selected 20 pairs of observed variables and introduced a hidden confounder between each pair, and for another 20 pairs of observed variables, we added a hidden intermediate variable. We generated $50$ random graphs in this way. For each random graph, we generated one dataset with the same process as in Section~\ref{sec:experiment_illustrative}. Each dataset contains $500$ samples. We ran the algorithms with significance level $\alpha = 0.05$, $0.1$, and $0.2$. The results are shown in Fig.~\ref{fig:ER_exp}. CAM-UV-X achieves performance comparable to that of CAM-UV.
\begin{figure}[!ht] 
\centering
\includegraphics[width = 0.6\columnwidth]{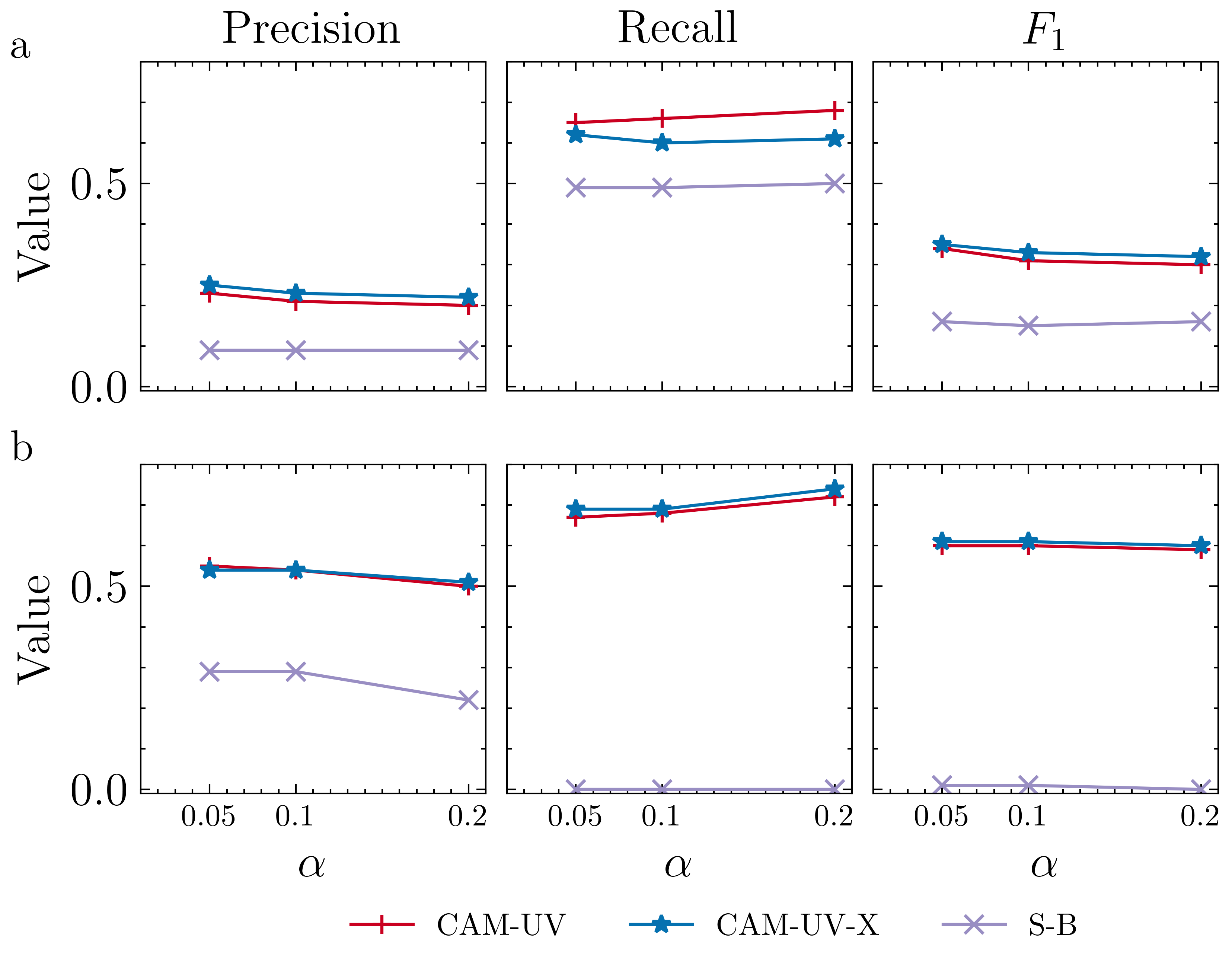}
\caption{Performance in ER random graphs. a: identifying the adjacency matrix, b: identifying ancestor relationships.}\label{fig:ER_exp}
\end{figure}

\subsection{Details on the experiment with the sociological data}\label{appendix:sociology_data}

We removed outliers in the data by retaining samples in which the value of $x_2$ (son's income) is at or below the $0.95$ quantile of its sample distribution and the value of $x_6$ (number of siblings) is at or below the $0.95$ quantile of its sample distribution. The final number of samples is $1262$.

We ran CAM-UV and CAM-UV-X with $\alpha = 0.05$. For CAM-UV, we applied the following prior knowledge: each of $x_2$ (son's income), $x_4$ (son's occupation), and $x_5$ (son's education) cannot be parents of either $x_1$ (father's occupation), $x_3$ (father's education), or $x_6$ (number of siblings). We ran CAM-UV-X on the adjacency matrix estimated by CAM-UV, without any prior knowledge. 

\begin{thebibliography}{39}
\providecommand{\natexlab}[1]{#1}
\providecommand{\url}[1]{\texttt{#1}}
\expandafter\ifx\csname urlstyle\endcsname\relax
  \providecommand{\doi}[1]{doi: #1}\else
  \providecommand{\doi}{doi: \begingroup \urlstyle{rm}\Url}\fi

\bibitem[Adams et~al.(2021)Adams, Hansen, and Zhang]{NEURIPS2021_c0f6fb5d}
J.~Adams, N.~Hansen, and K.~Zhang (2021).
\newblock Identification of partially observed linear causal models: Graphical conditions for the non-Gaussian and heterogeneous cases.
\newblock \emph{Advances in Neural Information Processing Systems}.

\bibitem[Ashman et~al.(2023)Ashman, Ma, Hilmkil, Jennings, and Zhang]{ashman2023causal}
M.~Ashman, C.~Ma, A.~Hilmkil, J.~Jennings, and C.~Zhang (2023).
\newblock Causal reasoning in the presence of latent confounders via neural {ADMG} learning.
\newblock \emph{International Conference on Learning Representations}.

\bibitem[Barab{\'{a}}si and Albert(1999)]{ba_model}
A.-L. Barab{\'{a}}si and R.~Albert (1999).
\newblock Emergence of scaling in random networks.
\newblock \emph{Science}, 286\penalty0 (5439):\penalty0 509--512.

\bibitem[Bhattacharya et~al.(2021)Bhattacharya, Nagarajan, Malinsky, and Shpitser]{Bhattacharya21ABC}
R.~Bhattacharya, T.~Nagarajan, D.~Malinsky, and I.~Shpitser (2021).
\newblock Differentiable causal discovery under unmeasured confounding.
\newblock\emph{International Conference on Artificial Intelligence and Statistics}.

\bibitem[Budhathoki et~al.(2022)Budhathoki, Minorics, Bl\"{o}baum, and Janzing]{pmlr-v162-budhathoki22a}
K.~Budhathoki, L.~Minorics, P.~Bl\"{o}baum, and D.~Janzing (2022).
\newblock Causal structure-based root cause analysis of outliers.
\newblock \emph{International Conference on Machine Learning}.

\bibitem[B{\"u}hlmann et~al.(2014)B{\"u}hlmann, Peters, and Ernest]{Buhlmann14CAM}
P.~B{\"u}hlmann, J.~Peters, and J.~Ernest (2014).
\newblock {CAM}: Causal additive models, high-dimensional order search and penalized regression.
\newblock \emph{Annals of Statistics}, 42\penalty0 (6):\penalty0 2526--2556.

\bibitem[Chen and Guestrin(2016)]{xgboost}
T.~Chen and C.~Guestrin (2016).
\newblock {XGBoost}: A scalable tree boosting system.
\newblock \emph{ACM SIGKDD International Conference on Knowledge Discovery and Data Mining}.


\bibitem[Chen et~al.(2024)Chen, Huang, Cai, Hao, and Zhang (2024)]{10.1609/aaai.v38i18.30017}
W.~Chen, Z.~Huang, R.~Cai, Z.~Hao, and K.~Zhang (2024).
\newblock Identification of causal structure with latent variables based on higher order cumulants.
\newblock \emph{AAAI Conference on Artificial Intelligence}.

\bibitem[Chickering(2002)]{Chickering2002}
D.~M. Chickering (2002).
\newblock Optimal structure identification with greedy search.
\newblock \emph{Journal of Machine Learning Research}, 3\penalty0 (Nov):\penalty0 507--554.

\bibitem[Duncan et~al.(1972)Duncan, Featherman, and Duncan]{Duncan72book}
O.~D. Duncan, D.~L. Featherman, and B.~Duncan (1972).
\newblock \emph{Socioeconomic Background and Achievement}.
\newblock Seminar Press, New York.

\bibitem[Erd\"{o}s and R\'{e}nyi(1959)]{erdos59a}
P.~Erd\"{o}s and A.~R\'{e}nyi (1959).
\newblock On random graphs {I}.
\newblock \emph{Publicationes Mathematicae Debrecen}, 6:\penalty0 290.

\bibitem[Ge et~al.(2025)Ge, Cai, Wan, Xu, and Song]{causal_decision_making}
L.~Ge, H.~Cai, R.~Wan, Y.~Xu, and R.~Song (2025).
\newblock A review of causal decision making.
\newblock \emph{arXiv}, 	arXiv:2502.16156.

\bibitem[Glymour et~al.(2019)Glymour, Zhang, and Spirtes]{glymour2019review}
C.~Glymour, K.~Zhang, and P.~Spirtes (2019).
\newblock Review of causal discovery methods based on graphical models.
\newblock \emph{Frontiers in Genetics}, 10:\penalty0 524.

\bibitem[Gretton et~al.(2007)Gretton, Fukumizu, Teo, Song, Sch\"{o}lkopf, and Smola]{gretton_hsic}
A.~Gretton, K.~Fukumizu, C.~Teo, L.~Song, B.~Sch\"{o}lkopf, and A.~Smola (2007).
\newblock A kernel statistical test of independence.
\newblock \emph{Advances in Neural Information Processing Systems}.

\bibitem[Hastie and Tibshirani(1986)]{hastie_GAM}
T.~Hastie and R.~Tibshirani (1986).
\newblock {Generalized additive models}.
\newblock \emph{Statistical Science}, 1\penalty0 (3):\penalty0 297 -- 310.

\bibitem[Hoyer et~al.(2008{\natexlab{a}})Hoyer, Janzing, Mooij, Peters, and Sch\"{o}lkopf]{Hoyer09NIPS}
P.~O. Hoyer, D.~Janzing, J.~Mooij, J.~Peters, and B.~Sch\"{o}lkopf (2008a).
\newblock Nonlinear causal discovery with additive noise models.
\newblock \emph{{Advances in Neural Information Processing Systems}}.

\bibitem[Hoyer et~al.(2008{\natexlab{b}})Hoyer, Shimizu, Kerminen, and Palviainen]{Hoyer08IJAR}
P.~O. Hoyer, S.~Shimizu, A.~Kerminen, and M.~Palviainen (2008b).
\newblock Estimation of causal effects using linear non-{Gaussian} causal models with hidden variables.
\newblock \emph{International Journal of Approximate Reasoning}, 49\penalty0 (2):\penalty0 362--378.

\bibitem[Maeda and Shimizu(2020)]{Maeda20AISTATS}
T.~N. Maeda and S.~Shimizu (2020).
\newblock {RCD}: Repetitive causal discovery of linear non-{G}aussian acyclic models with latent confounders.
\newblock \emph{International Conference on Artificial Intelligence and Statistics}.

\bibitem[Maeda and Shimizu(2021)]{maeda21a}
T.~N. Maeda and S.~Shimizu (2021).
\newblock Causal additive models with unobserved variables.
\newblock \emph{Conference on Uncertainty in Artificial Intelligence}.

\bibitem[Mani et~al.(2006)Mani, Spirtes, and Cooper]{Mani2006UAI}
S.~Mani, P.~Spirtes, and G.~F. Cooper (2006).
\newblock A theoretical study of Y structures for causal discovery.
\newblock \emph{Conference on Uncertainty in Artificial Intelligence}.

\bibitem[Ogarrio et~al.(2016)Ogarrio, Spirtes, and Ramsey]{Ogarrio16Hybrid}
J.~M. Ogarrio, P.~Spirtes, and J.~Ramsey (2016).
\newblock A hybrid causal search algorithm for latent variable models.
\newblock \emph{Conference on Probabilistic Graphical Models}.

\bibitem[Peters et~al.(2011)Peters, Mooij, Janzing, and Sch\"{o}lkopf]{causal_functional}
J.~Peters, J.~M. Mooij, D.~Janzing, and B.~Sch\"{o}lkopf (2011).
\newblock Identifiability of causal graphs using functional models.
\newblock \emph{Conference on Uncertainty in Artificial Intelligence}.

\bibitem[Runge(2018)]{Runge_CMIknn}
J.~Runge (2018).
\newblock Conditional independence testing based on a nearest-neighbor estimator of conditional mutual information.
\newblock \emph{International Conference on Artificial Intelligence and Statistics}.

\bibitem[Salehkaleybar et~al.(2020)Salehkaleybar, Ghassami, Kiyavash, and Zhang]{Salehkaleybar2020learning}
S.~Salehkaleybar, A.~Ghassami, N.~Kiyavash, and K.~Zhang (2020).
\newblock Learning linear non-{G}aussian causal models in the presence of latent variables.
\newblock \emph{Journal of Machine Learning Research}, 21\penalty0 (39):\penalty0 1--24.

\bibitem[Sch\"{o}lkopf(2022)]{causality_for_ML}
B.~Sch\"{o}lkopf (2022).
\newblock \emph{Causality for Machine Learning}.
\newblock Association for Computing Machinery, New York, NY, USA, 1st edition.

\bibitem[Schultheiss and B{{\"u}}hlmann(2024)]{Schultheiss_2024}
C.~Schultheiss and P.~B{{\"u}}hlmann (2024).
\newblock Assessing the overall and partial causal well-specification of nonlinear additive noise models.
\newblock \emph{Journal of Machine Learning Research}, 25\penalty0 (159):\penalty0 1--41.

\bibitem[Servén and Brummitt(2018)]{Servan18pyGAM}
D.~Servén and C.~Brummitt (2018).
\newblock pyGAM: Generalized additive models in Python.
\newblock URL \url{https://doi.org/10.5281/zenodo.1208723}.

\bibitem[Shimizu et~al.(2006)Shimizu, Hoyer, Hyv{\"a}rinen, and Kerminen]{Shimizu06JMLR}
S.~Shimizu, P.~O. Hoyer, A.~Hyv{\"a}rinen, and A.~Kerminen (2006).
\newblock A linear non-{Gaussian} acyclic model for causal discovery.
\newblock \emph{Journal of Machine Learning Research}, 7:\penalty0 2003--2030.

\bibitem[Shimizu et~al.(2011)Shimizu, Inazumi, Sogawa, Hyv{\"a}rinen, Kawahara, Washio, Hoyer, and Bollen]{Shimizu11JMLR}
S.~Shimizu, T.~Inazumi, Y.~Sogawa, A.~Hyv{\"a}rinen, Y.~Kawahara, T.~Washio, P.~O. Hoyer, and K.~Bollen (2011).
\newblock {DirectLiNGAM}: A direct method for learning a linear non-{G}aussian structural equation model.
\newblock \emph{Journal of Machine Learning Research}, 12:\penalty0 1225--1248.

\bibitem[Spirtes and Glymour(1991)]{Spirtes91PC}
P.~Spirtes and C.~Glymour (1991).
\newblock An algorithm for fast recovery of sparse causal graphs.
\newblock \emph{Social Science Computer Review}, 9:\penalty0 67--72.

\bibitem[Spirtes et~al.(1995)Spirtes, Meek, and Richardson]{Spirtes95FCI}
P.~Spirtes, C.~Meek, and T.~Richardson (1995).
\newblock Causal inference in the presence of latent variables and selection bias.
\newblock \emph{{Conference on Uncertainty in Artificial Intelligence}}.

\bibitem[Spirtes et~al.(2001)Spirtes, Glymour, and Scheines]{Spirtes2001}
P.~Spirtes, C.~Glymour, and R.~Scheines (2001).
\newblock \emph{Causation, Prediction, and Search}.
\newblock MIT Press, 2nd edition.

\bibitem[Suzuki et~al.(2026)]{Suzuki2026}
H.~Suzuki, K.~Kanamori, T.~Takagi, T.~Pham, T.N.~Maeda, and S.~Shimizu (2026).
\newblock {I-CAM-UV}: Integrating causal graphs over non-identical variable sets
using causal additive models with unobserved variables.
\newblock \emph{{AAAI} Conference on Artificial Intelligence}.

\bibitem[Sz{\'e}kely et~al.(2007)Sz{\'e}kely, Rizzo, and Bakirov]{GPDC}
G.~J. Sz{\'e}kely, M.~L. Rizzo, and N.~K. Bakirov (2007).
\newblock {Measuring and testing dependence by correlation of distances}.
\newblock \emph{The Annals of Statistics}, 35\penalty0 (6):\penalty0 2769 -- 2794.

\bibitem[Tashiro et~al.(2014)Tashiro, Shimizu, Hyv{\"a}rinen, and Washio]{Tashiro14NECO}
T.~Tashiro, S.~Shimizu, A.~Hyv{\"a}rinen, and T.~Washio (2014).
\newblock {ParceLiNGAM}: A causal ordering method robust against latent confounders.
\newblock \emph{Neural Computation}, 26\penalty0 (1):\penalty0 57--83.

\bibitem[Tramontano et~al.(2024)Tramontano, Kivva, Salehkaleybar, Drton, and Kiyavash]{pmlr-v235-tramontano24a}
D.~Tramontano, Y.~Kivva, S.~Salehkaleybar, M.~Drton, and N.~Kiyavash (2024).
\newblock Causal effect identification in {L}i{NGAM} models with latent confounders.
\newblock \emph{International Conference on Machine Learning}.

\bibitem[Wang and Drton(2023)]{wang_2023}
Y.~S. Wang and M.~Drton (2023).
\newblock Causal discovery with unobserved confounding and non-Gaussian data.
\newblock \emph{Journal of Machine Learning Research}, 24(271):1-61.

\bibitem[Yokoyama et~al.(2025)Yokoyama, Shingaki, Nishino, Shimizu, and Pham]{yokoyama_2025}
H.~Yokoyama, R.~Shingaki, K.~Nishino, S.~Shimizu, and T.~Pham (2025).
\newblock Causal-discovery-based root-cause analysis and its application in time-series prediction error diagnosis.
\newblock \emph{International Joint Conference on Neural Networks}.

\bibitem[Zhang et~al.(2010)Zhang, Sch{\"o}lkopf, and Janzing]{Zhang10UAI-GP}
K.~Zhang, B.~Sch{\"o}lkopf, and D.~Janzing (2010).
\newblock Invariant {Gaussian} process latent variable models and application in causal discovery.
\newblock \emph{{Conference on Uncertainty in Artificial Intelligence}}.

\bibitem[Zheng et~al.(2018)Zheng, Aragam, Ravikumar, and Xing]{Zheng18Neurips}
X.~Zheng, B.~Aragam, P.~K. Ravikumar, and E.~P. Xing (2018).
\newblock {DAG}s with {NO TEARS}: Continuous optimization for structure learning.
\newblock \emph{Advances in Neural Information Processing Systems}.
\end{thebibliography}
\end{document}